\def\isarxivversion{1} %%% for icml submission version, we comment this line

\ifdefined\isarxivversion
\documentclass[11pt]{article}
\else
\documentclass{article}
\fi

\usepackage{microtype}
\usepackage{graphicx}
\usepackage{subfigure}
\usepackage{booktabs} % for professional tables
\ifdefined\isarxivversion
\usepackage{algorithm}
\usepackage{algorithmic}
\fi

\usepackage[utf8]{inputenc}
\usepackage{amsmath}
\usepackage{amsthm}
\usepackage{amsfonts}
\usepackage{float}
\usepackage{graphicx}
\usepackage{bm}
\usepackage{amssymb}
\usepackage{bbm}
\usepackage{color}
\usepackage{mathtools}
\usepackage{nicefrac}
\usepackage[super]{nth}

\usepackage{hyperref}

\DeclareMathOperator{\proj}{proj}
\DeclareMathOperator{\poly}{poly}

\DeclareMathOperator{\out}{out}
\DeclareMathOperator{\test}{test}
\DeclareMathOperator{\tr}{tr}

\DeclareMathOperator*{\argmax}{arg\,max}
\DeclareMathOperator*{\argmin}{arg\,min}

\newcommand{\wh}{\widehat}
\newcommand{\wt}{\widetilde}

\newcommand{\eps}{\epsilon}
\newcommand{\N}{\mathcal{N}}
\newcommand{\R}{\mathbb{R}}
\renewcommand{\P}{\mathbb{P}}

\renewcommand{\varepsilon}{\epsilon}
\renewcommand{\tilde}{\wt}
\renewcommand{\hat}{\wh}
\renewcommand{\eps}{\epsilon}

\newtheorem{propo}{Proposition}
\newtheorem{thm}[propo]{Theorem}
\newtheorem{asmp}{Assumption}
\newtheorem{theorem}{Theorem}[section]
\newtheorem{lemma}[theorem]{Lemma}
\newtheorem{definition}[theorem]{Definition}

\newtheorem{proposition}[theorem]{Proposition}
\newtheorem{corollary}[theorem]{Corollary}

\newtheorem{fact}[theorem]{Fact}
\newtheorem{remark}[theorem]{Remark}

\newcommand{\round}[1]{\left( #1 \right)}
\newcommand{\curly}[1]{\left\lbrace #1 \right\rbrace}
\newcommand{\squarebrack}[1]{\left\lbrack #1 \right\rbrack}

\newcommand{\sumi}[2]{\sum\limits_{i=#1}^{#2}}
\newcommand{\sumj}[2]{\sum\limits_{j=#1}^{#2}}

\newcommand{\abs}[1]{\left\lvert #1 \right\rvert}
\newcommand{\norm}[1]{\left\lVert#1\right\rVert}
\newcommand{\esqnorm}[1]{\left\lVert#1\right\rVert_2^2}
\newcommand{\enorm}[1]{\left\lVert#1\right\rVert_2}

\newcommand{\zero}{\mathbf{0}}
\newcommand{\one}{\mathbf{1}}

\newcommand{\e}{\mathbf{e}}

\newcommand{\p}{\mathbf{p}}
\newcommand{\q}{\mathbf{q}}

\newcommand{\s}{\mathbf{s}}

\newcommand{\uvec}{\mathbf{u}}
\newcommand{\vvec}{\mathbf{v}}
\newcommand{\w}{\mathbf{w}}
\newcommand{\x}{\mathbf{x}}

\newcommand{\y}{\mathbf{y}}
\newcommand{\z}{\mathbf{z}}

\newcommand{\A}{\mathbf{A}}
\newcommand{\B}{\mathbf{B}}

\newcommand{\Hmat}{\mathbf{H}}
\newcommand{\I}{\mathbf{I}}

\newcommand{\M}{\mathbf{M}}

\newcommand{\Smat}{\mathbf{S}}
\newcommand{\U}{\mathbf{U}}
\newcommand{\V}{\mathbf{V}}
\newcommand{\W}{\mathbf{W}}
\newcommand{\X}{\mathbf{X}}

\newcommand{\Z}{\mathbf{Z}}

\newcommand{\SIGMA}{\mathbf{\Sigma}}

\newcommand{\Ccal}{\mathcal{C}}
\newcommand{\Dcal}{\mathcal{D}}
\newcommand{\Ecal}{\mathcal{E}}

\newcommand{\Ncal}{\mathcal{N}}
\newcommand{\Ocal}{\mathcal{O}}
\newcommand{\Pcal}{\mathcal{P}}

\newcommand{\Xcal}{\mathcal{X}}

\newcommand{\betavec}{\boldsymbol{\beta}}

\newcommand{\muvec}{\boldsymbol{\mu}}

\newcommand{\BigO}[1]{\mathcal{O}\round{#1}}

\newcommand{\Rd}[1]{\mathbb{R}^{#1}}
\newcommand{\Natural}{\mathbb{N}}

\DeclareMathOperator*{\myE}{\mathbb{E}}
\newcommand{\E}[1]{\myE\squarebrack{#1}}
\newcommand{\Exp}[2]{\myE_{#1}\squarebrack{#2}}
\newcommand{\Prob}[1]{\mathbb{P}\squarebrack{#1}}
\newcommand{\Var}[1]{\mathrm{Var}\squarebrack{#1}}

\newcommand{\inv}[1]{\frac{1}{#1}}
\newcommand{\indicator}[1]{\mathbbm{1}\curly{#1}}

\definecolor{b2}{RGB}{51,153,255}
\definecolor{mygreen}{RGB}{80,180,0}
\newcommand{\Zhao}[1]{{\color{mygreen}[Zhao: #1]}}
\newcommand{\Weihao}[1]{{\color{cyan}[Weihao: #1]}}
\newcommand{\Raghav}[1]{{\color{blue}[Raghav: #1]}}

\ifdefined\isarxivversion
\usepackage[margin=1in]{geometry}
\else 

\fi

\ifdefined\isarxivversion
\usepackage{natbib}
\else
\usepackage{icml2020}
\fi

\begin{document}

\ifdefined\isarxivversion

\title{Meta-learning for mixed  linear regression }

\date{}

%%% Zhao: I just the authors as alphabetic, I am happy to change to contribution order (putting student first)
\author{
Weihao Kong\thanks{\texttt{kweihao@gmail.com}. University of Washington}
\and
Raghav Somani\thanks{\texttt{raghavs@cs.washington.edu}. University of Washington}
\and
Zhao Song\thanks{\texttt{zhaos@ias.edu}. Princeton University and Institute for Advanced Study}
\and
Sham Kakade\thanks{\texttt{sham@cs.washington.edu}.  University of Washington}
\and
Sewoong Oh\thanks{\texttt{sewoong@cs.washington.edu}. University of Washington}
}

\else

% The \icmltitle you define below is probably too long as a header.
% Therefore, a short form for the running title is supplied here:
\icmltitlerunning{
Meta-learning for mixed  linear regression}

\twocolumn[
\icmltitle{Meta-learning for mixed  linear regression}




\icmlsetsymbol{equal}{*}

\begin{icmlauthorlist}
\icmlauthor{Weihao Kong}{uw}
\icmlauthor{Raghav Somani}{uw}
\icmlauthor{Zhao Song}{pri}
\icmlauthor{Sewoong Oh}{uw}
\end{icmlauthorlist}

\icmlaffiliation{uw}{University of Washington, Seattle, Washington, USA}
\icmlaffiliation{pri}{Princeton University/Institute for Advanced Study}
% \icmlaffiliation{goo}{Googol ShallowMind, New London, Michigan, USA}
% \icmlaffiliation{ed}{School of Computation, University of Edenborrow, Edenborrow, United Kingdom}

\icmlcorrespondingauthor{Weihao Kong}{kweihao@gmail.com}
\icmlcorrespondingauthor{Raghav Somani}{raghavs@cs.washington.edu}
\icmlcorrespondingauthor{Zhao Song}{zhaos@ias.edu}
\icmlcorrespondingauthor{Sewoong Oh}{sewoong@cs.washington.edu}

\icmlkeywords{Machine Learning, ICML}

\vskip 0.3in
]

\fi




\ifdefined\isarxivversion
\begin{titlepage}
  \maketitle
  \begin{abstract}
% weihao
%In many setting lots of data from heterogeneous tasks, each task has a small amount of data. Learning is impossible for each task separately. Meta-learning address this setting.
%Despite empirical success, few theoretically results are known about to what extent can the large number of prior data sources help with inference for the new task.

 %Meta-learning addresses an important challenge i

In modern supervised learning, there are a large number of tasks, but many of them are associated with only a small amount of labelled data. These include data from medical image processing and robotic interaction. 
Even though each individual task cannot be meaningfully trained in isolation, 
 one seeks to meta-learn across the tasks from past experiences by exploiting some similarities. 
We study a  fundamental question of interest: 
When can abundant tasks with small data compensate for lack of tasks with big data? 
 We focus on a canonical scenario where each task is drawn from a mixture of $k$ linear regressions, 
 and identify sufficient conditions for such a graceful exchange to hold; The total number of examples necessary with only small data tasks scales similarly as when big data tasks are available.
 To this end, we introduce a novel spectral approach and 
 show that we 
 can efficiently utilize small data tasks with the help of $\tilde\Omega(k^{3/2})$ medium data tasks each with $\tilde\Omega(k^{1/2})$ examples.

% This is achieved by introducing a novel spectral algorithm and analyzing the sample complexity involved in each step.  
%We study the classical mixed liner regression model in the meta learning setting...proposed efficient and practical algorithm for parameter estimation.... comparing to previous work on mixed linear regression, our algorithm achieves near optimal complexity by using a small number of tasks that has relatively large amount of data. use heavy tasks and light tasks, fit for many practical scenario... 
%we characterized the benefit one can achieve through meta learning, and the interplay between the number of priors tasks, and the informatio per task...

  \end{abstract}
  \thispagestyle{empty}
\end{titlepage}

\else

\begin{abstract}

\end{abstract}

\fi

\section{Introduction}
\label{sec:intro} 

Recent advances in machine learning highlight successes on a small set  of tasks where a large number of labeled examples have been collected and exploited. These include image classification with 1.2 million labeled examples~\cite{deng2009imagenet} and French-English machine translation with 40 million paired sentences~\cite{bojar2014findings}. For common tasks, however, collecting clean labels is costly, as they require human expertise (as in medical imaging) or  physical interactions (as in robotics), for example. Thus collected real-world datasets follow a long-tailed distribution, in which a dominant set of tasks only have a small number of training examples~\cite{wang2017learning}. 

Inspired by human ingenuity in quickly solving novel problems by leveraging prior experience,  {\em meta-learning}  approaches aim to jointly learn from past experience to quickly adapt to new tasks with little available data \cite{Sch87,TP12}. This has had a significant impact in few-shot supervised learning, where each task is associated with only a few training examples. By leveraging structural similarities among those tasks, one can achieve accuracy far greater than what can be achieved for each task in isolation \cite{FAL17,RL16,KZS15,OLL18,triantafillou2019meta,rusu2018meta}. 
%This is critical in democratizing machine learning  %where small data tasks can help each other in achieving the accuracy of big data tasks. 
The success of such approaches hinges on the following fundamental question: When can we jointly train small data tasks to achieve the accuracy of  large data tasks?

% data has long tail - cannot be learned in isolation, remedy is meta-learning: leveraging past experience to make predictions on new tasks more accurate. empirical successes.  

% fundamental question: when can it compaensate for ? how many small data is equivalent to alrge data? 

% we address this in a canonical setting of linear regressions. regression parameter w_i, p_i, for some finite set of k.  

We investigate this trade-off under a canonical scenario where the tasks are linear regressions in $d$-dimensions and the regression parameters are drawn i.i.d.~from a discrete set of a support size $k$. 
Although widely studied, existing literature addresses the scenario  where 
 %each task only has one example (known as mixed linear regression) and where 
all tasks have the same fixed number of examples. %(known as multi-task learning). 
We defer formal comparisons to Section~\ref{sec:related}. 

On one extreme, when large training data of sample size $\Omega(d)$ is available, each task can easily be learned in isolation; here, $\Omega(k\log k)$ such tasks are sufficient to learn all $k$ regression parameters. This is illustrated by a solid circle in Figure~\ref{fig:intro}. On the other extreme, when each task has only one example, existing approaches require exponentially many tasks (see Table~\ref{tab:related}). This is illustrated by a solid square.

Several aspects of  few-shot supervised learning makes training linear models  challenging. The number of training examples varies significantly across tasks, all of which are significantly smaller than the dimension of the data $d$. The number of tasks are also limited, which  restricts 
any algorithm with exponential sample complexity. 
An example distribution of 
such heterogeneous tasks is illustrated in Figure~\ref{fig:intro} with a bar graph in blue, where both the solid circle and square are far outside of the regime covered by the typical  distribution of tasks. 

%should be close to $kd$, which is the information theoretic lower bound on how many examples are necessary to learn the $k$ linear models each in $d$ dimensions. 

\begin{figure}[ht] 
\begin{center} 
\centerline{\includegraphics[width=.5\columnwidth]{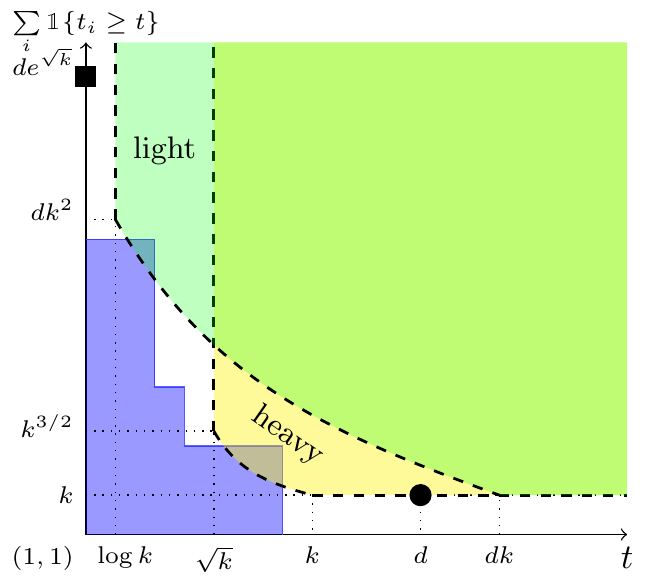}}
\caption{ 
Realistic pool of meta-learning tasks do not include large data tasks (circle) or extremely large number of small data tasks (square), where existing approaches achieve high accuracy. 
The horizontal axis denotes the number of examples $t$ per task, and the vertical axis denotes the number of tasks in the pool that have at least $t$ examples. The proposed approach succeeds whenever any point in the light (green) region, and any point in the heavy (yellow) region are both covered by the blue bar graph, as is in this example. The blue graph summarizes the pool of tasks in hand, illustrating the cumulative count of tasks with more than $t$ examples.
We ignore constants and ${\rm poly}\log$ factors.}
\label{fig:intro} 
\end{center} 
\end{figure} 

In this data scarce regime,  we show that we can still efficiently achieve any desired accuracy in estimating the meta-parameters defining the meta-learning problem. This is shown in the informal version of our main result in Corollary~\ref{cor:main}. As long as we have enough number of {\em light tasks} each with $t_L=\tilde\Omega(1)$ examples, we can achieve any accuracy with the help of a small number of {\em heavy tasks} each with $t_H= \tilde{\Omega}(\sqrt{k})$ examples. We only require the total number of examples that we have jointly across all light tasks to be of order $t_Ln_L=\tilde\Omega(dk^2)$; the number of light tasks $n_L$ and the number of examples per task $t_L$ trade off gracefully. This is illustrated by the green region in Figure~\ref{fig:intro}. Further, we only need a small number of heavy tasks  with $t_Hn_H=\tilde\Omega(k^{3/2})$, shown in the yellow region. As long as the cumulative count of tasks in blue graph intersects with the light (green) and heavy (yellow) regions, we can recover the meta-parameters accurately. 

%\Sewoong{
%Kaggle's diabetic retinopathy detection dataset 35k labeled images
%Adaptive epilepsy treatment with RL, Guez et al. 08 <1hour data 
%Learning for robotic manipulation Finn et al. '16 <15min data
%}

%{\bf Meta-learning.}\cite{FAL17,LZC17,FXL18,BHT18,ZWL18,ZSK19}Computational and memory burden of evaluating the gradient for the inner loop optimization was solved \cite{RFK19} by making the inner loop implicit. 

%There is a fundamental question of interest: How large $n$ needs to be, such that given a new person with $t_{test}$ samples our algorithm is close to the optimal algorithm that knows $\Wcal$? \Sewoong{perhaps we should rephrase this.}

\begin{corollary}[Special case of Theorem~\ref{thm:main_reg_est}, informal]\hfill
\label{cor:main}
%\begin{enumerate}
%\item Suppose $\Delta = \Omega(1)$, and $p_{\rm min} = \Omega(1/k)$. 
Given two batch of samples, the first batch with
\[
t_L = \tilde\Omega(1)\;,\;t_Ln_L = \tilde\Omega\round{dk^2},
\]
and the second batch with 
\[
t_H = \tilde\Omega\big(\sqrt{k}\big)\;,\;t_Hn_H = \tilde\Omega\big(k^{2}\big),
\]
Algorithm~\ref{alg:estimate} estimates the meta-parameters up to any desired accuracy of $\BigO{1}$ with a high probability, under a certain assumptions on the meta-parameters.
%\begin{align*}
%\enorm{\hat{\w}_i-\w_i} &\le \eps s_i\; , \\ 
%\abs{\hat{s}_i^2-s_i^2}&\le  \frac{\eps}{\sqrt{d}}s_i^2\; ,\quad\text{and}\\
%\abs{\hat{p}_i-p_i} &\le \frac{\eps}{\sqrt{d}} p_i.
%\end{align*}
%\item
%Given a new task $\cal T^{\rm new}$ with $\tau\ge \Theta\round{\log (k/\delta)}$ samples $\mathcal{D} = \{ \x_i,y_i\}_{i=1}^\tau$. 
%Then, the MAP/posterior mean estimator $\hat{\beta}$ using the estimated parameters approximates the true regression vector $\beta^{\rm new}$ with expected $\ell_2$ error bounded as
%\begin{align*}
%&\mathbb{E}_{\cal T^{\rm new} \sim \P{(\cal T)}}\mathbb{E}_{\cal D \sim \cal T^{\rm new}}\squarebrack{\esqnorm{\hat{\beta}({\cal D})-\beta^{\rm new}}}\le \delta+\eps^2
%\end{align*}
%\end{enumerate}
\end{corollary}

We design a novel spectral approach inspired by \cite{vempala2004spectral} that first learns a subspace using the light tasks, and then clusters the heavy tasks in the projected space. To get the desired tight bound on the sample complexity, we improve upon a perturbation bound from \cite{li2018learning}, and borrow techniques from recent advances in property testing in \cite{kong2019sublinear}.

%\Sewoong{Meta-learning addresses a challenging scenario of a data poor regime where $(a)$ we might have a large number of tasks $n$, but each task has only a small number $t\ll d$ of  labelled examples; and $(b)$ the model for a new task needs to be learned from an even smaller  $\tau \ll t$ number of labelled examples.}

\section{Problem formulation and notations} 
\label{sec:model} 

There are two perspectives
on approaching meta-learning: 
 optimization based 
 \cite{LZC17,BHT18,ZWL18,ZSK19,RFK19}, and probabilistic \cite{grant2018recasting,FXL18,KYD18,HSP18}. 
 Our approach is motivated by the probabilistic view and we present a brief preliminary in Section~\ref{sec:metalearning}. 
 %, optimization based methods can also be derived from the probabilistic perspective of \cite{grant2018recasting}. 
%In Section~\ref{sec:metalearning}, we present a brief preliminary on the  probabilistic view of meta-learning. 
%, and refer to~\cite{grant2018recasting} for a comprehensive introduction.
In Section~\ref{sec:linear}, we  present a simple but  
 canonical scenario where the tasks are linear regressions, which is the focus of this paper. 

\subsection{Review of probabilistic view on meta-learning}
\label{sec:metalearning}

A standard meta-training for few-shot supervised learning assumes that  
we are given a collection of $n$ meta-training tasks $\{{\cal T}_i\}_{i=1}^n$  drawn from some distribution $\P\round{{\cal T}}$.  
Each task is associated with a  dataset of size $t_i$, collectively denoted as a {\em meta-training dataset} 
${\cal D}_{\rm  meta\text{-} train}=\curly{ \{(\x_{i,j},y_{i,j})\in \Rd{d} \times \Rd{}\}_{j\in[t_i]} }_{i\in[n]}$. 
Exploiting  some structural similarities 
in $\P{({\cal T})}$, the goal is to train 
a model for a new task ${\cal T}^{\rm new}$, coming from $\P\round{{\cal T}}$, from a small  amount of {\em training dataset} ${\cal D}=\curly{(\x^{\rm new}_j,y^{\rm new}_j)}_{j\in[\tau]}$.

%We use $\w_1,\ldots, \w_k\in \Rd{d}$ to denote . Let $\Wcal$ be the uniform distribution supported on $\curly{\w_1,\ldots, \w_k}$. Assuming that for each %person $i\in [n]$, $\beta_i\sim \Wcal$ and $\x_{i,j}\sim \Ncal\round{\zero,\I_d}$, $y_{i,j} = \beta_i^\top \x_{i,j}+\epsilon_{i,j}$ where $\eps_{i,j}\sim \Ncal(0, \sigma_i^2)$ for all $j\in [t]$. 
Each  task ${\cal T}_i$ is associated with 
a {\em model parameter} $\phi_i$, where the meta-training data is independently drawn from: $(\x_{i,j},y_{i,j}) \sim {\mathbb P}_{\phi_i}(y | \x) {\mathbb P}(\x)$ for all $j\in[t_i]$. 
The prior distribution of the 
tasks, and hence the model parameters, is 
fully characterized by a 
{\em meta-parameter} $\theta$ such that $\phi_i \sim {\mathbb P}_{\theta}(\phi)$.

Following the definition from \cite{grant2018recasting}, the {\em meta-learning problem} is defined as estimating the most likely meta-parameter given meta-training data by solving 
\begin{eqnarray}
    \theta^* \;\; \in \;\; \argmax_\theta \;\;\log \,\P{ ( \theta\,|\, {\cal D}_{\rm meta\text{-}data} )}\;,
    \label{eq:metalearning}
\end{eqnarray}
which is a special case of empirical Bayes methods for learning the prior distribution from data \cite{carlin2010bayes}. 
Once meta-learning is done, the model parameter of a newly arriving task can be 
estimated by a Maximum a Posteriori (MAP) estimator: 
\begin{eqnarray}
   \hat{\phi} \;\in\; \argmax_\phi
   \;\; \log \,\P{( \phi  \,|\, {\cal D}, \theta^*)}\;,
   \label{eq:learn_MAP}
\end{eqnarray}
or a Bayes optimal estimator:  
\begin{eqnarray}
   \hat{\phi} \;\in\; \argmin_\phi
   \;\; {\mathbb E}_{\phi'\sim \P{( \phi'  \,|\, {\cal D}, \theta^*)}} 
   [ \,\ell(\phi,\phi') \, ] \;,
   \label{eq:learn_Bayes}
\end{eqnarray}
for a choice of a loss function $\ell$. 
This estimated parameter is then used for predicting the label of a new data point $\x$ in task ${\cal T}^{\rm new}$ as 
\begin{eqnarray}
    \hat{y}\;\;\in\;\; \argmax_{y} \;\; {\mathbb P}_{\hat\phi}(y|\x) \;.
    \label{eq:predict}
\end{eqnarray}

\paragraph{General notations.} %We define a few notations for some mathematical objects we shall use. 
We define $\squarebrack{n}\coloneqq\curly{1,\ldots,n}\ \forall\ n\in\Natural$; $\norm{\x}_{p} \coloneqq \round{\sum_{x\in\x}\abs{x}^p}^{1/p}$ as the standard $\ell_p$-norm; and $B_{p,k}(\muvec,r) \coloneqq \curly{\x\in\Rd{k} \mid \norm{\x-\muvec}_{p} = r}$. 
%\forall\ p\geq 1$, $k\in\Natural$, $\muvec\in\Rd{k}$, and $r\in\R_+$;
$\Ncal\round{\muvec,\SIGMA}$ denotes the multivariate normal distribution with mean $\muvec\in\Rd{d}$ and covariance $\SIGMA\in\Rd{d\times d}$, and  $\indicator{E}$ denotes the indicator  of an event $E$.

% ------------------------------

\subsection{Linear regression with a discrete prior} 
\label{sec:linear}
In general, the meta-learning problem of \eqref{eq:metalearning} is computationally intractable and no statistical guarantees are known.  
To investigate the trade-offs involved, we assume  a simple but canonical scenario where the tasks are linear regressions: 
\begin{eqnarray}
     \x_{i,j}\sim\Pcal_{\x}\;,\;\;\;\;\; 
    y_{i,j} = \beta_i^\top \x_{i,j}+\epsilon_{i,j}\;,  
    \label{eq:model}
\end{eqnarray}
for the $i$-th task and $j$-th example. 
Each task is associated with a model parameter $\phi_i=\round{\beta_i\in{\mathbb R}^d, \sigma_i \in {\mathbb R}_+}$. % for all $i\in\squarebrack{n}$.
The noise $\eps_{i,j}$ is i.i.d.~as~$\eps_{i,j}\sim \Pcal_{\eps_i}$, and $\Pcal_{\eps_i}$ is a centered sub-Gaussian distribution with parameter $\sigma_i^2$. Without loss of generality, we assume that $\Pcal_{\x}$ is an isotropic (i.e. $\E{\x_{i,j}\x_{i,j}^\top} = \I_d$) centered sub-Gaussian distribution. If $\Pcal_{\x}$ is not isotropic, we assume there are large number of $\x_{i,j}$'s for whitening such that $\Pcal_{\x}$ is sufficiently close to isotropic.
% with second moment $\I_d$. 

We do not make any assumption on the prior of $\phi_i$'s 
other than that they come from a discrete distribution of a support size $k$. Concretely, the meta-parameter 
$\theta=\round{\W\in{\mathbb R}^{d\times k},\; \s\in{\mathbb R}^{k}_+,\; \p\in{\mathbb R}^{k}_+\cap B_{1,k}(\zero,1)}$
defines a discrete prior (which is also known as mixture of  linear experts \cite{CL13}) on $\phi_i$'s, 
where 
$\W=[\w_1,\ldots, \w_k] $ 
are the $k$ candidate model parameters, and  
$\s=[s_1,\ldots,s_k] $ are the $k$ candidate noise parameters. The $i$-th task is randomly chosen from one of the $k$ components from  distribution $\p$, denoted by $z_i\sim{\rm multinomial}(\p)$. The  training data is independently drawn from~\eqref{eq:model} for each $j\in[t_i]$  with $\beta_i=\w_{z_i}$ and $\sigma_i = s_{z_i}$.
%in the $(k-1)$-dimensional simplex. The  $t$ labeled data points are then drawn from 

 We want to characterize 
 the sample complexity of this meta-learning. This depends on how complex the 
 ground truths prior $\theta$ is. 
 This can be measured by the number of components $k$, the separation between the parameters $\W$, the minimum mixing probability $p_{\rm min}$, and the minimum positive eigen-value $\lambda_{\rm min}$ of the matrix $\sum_{j=1}^{k}p_j\w_j\w_j^\top$.

% which can be used directly to solve~\eqref{eq:learn}  as 
%\begin{eqnarray}
%\hat{i} &\coloneqq\; \argmax\limits_{i\in[k]} \inv{2\hat{s}_i^2}\sum\limits_{j\in[\tau]}-\round{\hat{\w}_i^\top\x^{\rm new}_j - y_j^{\rm new}}^2\nonumber\\
%&\qquad\qquad\qquad\qquad\qquad+\tau\log \frac{1}{\hat{s}_i} +\log\round{\hat{p}_i}, 
%\label{eq:learnlinear}
%\end{eqnarray}
%and let $\hat{\beta}=\hat{\w}_{\hat{i}}$ where $\hat{i}$ is the optimizer of the maximization in~\eqref{eq:learnlinear}. \Raghav{Should we also talk about Bayes optimal estimator here?} The prediction in 
%\eqref{eq:predict} simplifies to 
%$\hat{y} = \hat\beta^\top\x$, and is evaluated on 
%the prediction error:
%\begin{eqnarray}
%    \Exp{\round{\x^{\rm test},y^{\rm test}}}{\, \round{ \hat{\beta}^\top\x^{\rm test} - y^{\rm test} }^2 \,}\;,
%    \label{eq:performance}
%\end{eqnarray}
%where the test example $\round{\x^{\rm test},y^{\rm test}}$ is drawn from \eqref{eq:model} with 
%$\beta_i=\w_{z^{\rm new}}$ and $\sigma_i=s_{z^{\rm new}}$. 

{\bf Notations.} 
% $\rho_i^2 \coloneqq \Var{y_{i,1}|z_i} = s_{z_i}^2+\esqnorm{\w_{z_i}}$
We define $\rho_i \coloneqq \sqrt{s_{z_i}^2+\esqnorm{\w_{z_i}}}$ as the sub-Gaussian norm of a label $y_{i,j}$ in the $i$-th task, and $\rho^2 \coloneqq \max_{i} \rho_i^2$. %\Sewoong{variance and sub-Gaussian parameter is not the same.}
Without loss of generality, we assume $\rho=1$, which can be always  achieved by scaling the meta-parameters appropriately. We also define $p_{\rm min}\coloneqq\min_{j\in[k]}p_j$, and $\Delta\coloneqq \min_{i,j\in\squarebrack{k}, i\neq j}\enorm{\w_i-\w_j}$ and assume $p_{\rm min}, \Delta>0$. $\omega\in\R_+$ is such that two $n\times n$ matrices can be multiplied in $\BigO{n^\omega}$ time.
 %Without loss of generality, we assume $\enorm{\w_1}\geq \enorm{\w_2} \geq \cdots \geq \enorm{\w_k}$.\Raghav{Do we need this ordering anywhere?}

\section{Algorithm} 
\label{sec:algorithm}

We propose a novel  spectral approach (Algorithm~\ref{alg:estimate}) to solve the meta-learning linear regression,   
 consisting of three sub-algorithms: \textit{subspace estimation}, \textit{clustering}, and \textit{classification}.
These sub-algorithms require different types of tasks, depending on how many labelled examples are available. 
 
Clustering  requires 
{\em heay tasks}, where each task is associated with  many labelled examples, but we need a smaller number of such tasks. 
On the other hand, 
for subspace estimation and classification, {\em light tasks} 
are sufficient, where each task is associated with a few labelled examples. However, we need a large number of such tasks. 
In this section, we present the  intuition behind our algorithm design, and the types of tasks required. Precisely analyzing these requirements is the main contribution of this paper, to be presented in Section~\ref{sec:main}.

\subsection{Intuitions behind the algorithm design}

%We learn the subspace of the true parameters $\w_i$'s using the light tasks and perform clustering of the heavy tasks on that subspace. This provides estimates of the $\w_i$'s, which we use subsequently for classifying the remaining light tasks. As light tasks are abundant, we can get accurate estimates of the meta-parameters. These are used as priors on newly arriving tasks for prediction. 
We give a sketch of the algorithm below.   
Each step of meta-learning is spelled out in full detail in Section~\ref{sec:algorithm_detail}. 
This provides an estimated meta-parameter  $\hat{\theta}=\big(\hat{\W},\; \hat{\s},\; \hat{\p}\big)$.  
When a new task arrives, this can be readily applied to  solve for %the MAP estimation in \eqref{eq:learn_MAP} and Bayes optimal estimation in \eqref{eq:learn_Bayes}. 
prediction, as defined in Definition~\ref{def:predictor}.

\begin{algorithm}[h]
   \caption{}
   \label{alg:estimate}
\medskip
{\bf Meta-learning}
\begin{enumerate}
    \item {\em Subspace estimation.}~Compute subspace $\U$ which approximates ${\rm span}\curly{\w_1,\ldots, \w_k}$, with singular value decomposition. %\Raghav{$\U$ has been used throughout. Should we change it? we don't want to introduce $\U$ yet. Sure.}% using light tasks. 
    %, and project the data onto the subspace $\hat{\U}$.
    \item {\em Clustering.}~Project the heavy tasks onto the subspace of $\U$, perform distance-based $k$ clustering, and estimate $\tilde{\w}_i$ for each cluster.
    %of the heavy tasks in the $k$-dimensional subspace, and estimate $\hat{\w}_i$ for each cluster.
    \item {\em Classification.}~Perform likelihood-based classification of the light tasks using $\tilde{\w}_i$ estimated from the \textit{Clustering} step, and compute the more refined estimates $(\hat{\w}_i, \hat{s}_i, \hat{p}_i)$ of $(\w_i, s_i, p_i)$ for $i\in [k]$.
\end{enumerate}

{\bf Prediction}
\begin{enumerate}
    \item[4.] {\em Prediction.}~Perform MAP or Bayes optimal prediction using the estimated meta-parameter as a prior.
\end{enumerate}
\end{algorithm}

{\bf Subspace estimation.}  
The subspace spanned by the regression vectors, ${\rm span}\{\w_1,\ldots, \w_k\}$, can be easily estimated using data from the (possibly) light tasks with only $t_i \ge 2$. Using any two independent examples from the same task $(\x_{i,1}, y_{i,1}), (\x_{i,2}, y_{i,2})$, it holds that $\E{y_{i,1}y_{i,2}\x_{i,1}\x_{i,2}^\top} = \sum_{j=1}^k p_j \w_j\w_j^\top$. 
With a total of $\Omega(d \log d)$ such examples, the matrix $\sum_{j=1}^k p_j \w_j\w_j^\top$ can be accurately estimated under spectral norm, and so is the column space ${\rm span}\{\w_1,\ldots, \w_k\}$. We call this step \textit{subspace estimation}.

{\bf Clustering.}
Given an accurate estimation of the subspace ${\rm span}\{\w_1,\ldots, \w_k\}$, we can reduce the problem from a $d$-dimensional to a $k$-dimensional regression problem by  projecting $\x$ onto the subspace of $\U$. Tasks with $t_i=\Omega(k)$ examples can be individually trained as the unknown parameter is now in ${\mathbb R}^k$.
 %if we have $t = \Omega(k)$ examples for each task, we could solve the regression problem for each task independently. 
The fundamental question we address is: \textit{What can we do when $t_i=o(k)$?} 
We propose clustering such light tasks based on their estimates of the regression vector $\beta_i$'s, and jointly solve a single regression problem for each cluster.
 
To this end, we borrow techniques from recent advances in property estimation for linear regression. %The key issue here is how to obtain a good clustering when $t = o(k)$. 
Recently, in the contextual bandit setting, \citet{kong2019sublinear} proposed an estimator for the correlation between the linear regressors between a pair of datasets. Concretely, given two datasets  $\{\x_{1,j}, y_{1,j}\}_{j\in [t]}$ and $\{\x_{2,j}, y_{2,j}\}_{j\in [t]}$ whose true (unknown) regression vectors are $\beta_1$ and $\beta_2$, one can estimate $\esqnorm{\beta_1}$, $\esqnorm{\beta_2}$ and $\beta_1^\top \beta_2$ accurately with $t = \Ocal\big(\sqrt{d}\big)$. We use this technique to estimate  $\esqnorm{\beta_{i_2}-\beta_{i_2}}$, whose value can be used to check if the two tasks are in the same clusters. We cluster the  tasks with $t_i=\Omega\big(\sqrt{k}\big)$ into $k$ disjoint clusters. We call this step \textit{clustering}.

After clustering, resulting estimated $\tilde{\w}_i$'s have two sources of error: the error in the subspace estimation, and the error in the parameter estimation for each cluster. If we cluster more heavy tasks, we can reduce the second error but not the first.
We could increase the samples used in subspace estimation, but there is a more sample efficient way: classification. 
 
{\bf Classification.}
We start the classification step, once each cluster has enough (i.e.~$\Omega(k)$) datapoints to obtain a rough estimation of their corresponding regression vector. In this regime, we have  $\BigO{1}$ error in the estimated $\tilde{\w}_i$'s. This is sufficient for us to add more datapoints to grow each of the clusters. When enough data points are accumulated (i.e.~$\tilde\Omega(d)$ for each cluster), then we can achieve any desired accuracy with this larger set of accurately classified tasks. This separation of the roles of the three sub-algorithms is critical in achieving the tightest sample complexity. 

In contrast to the necessary condition of  $t_i=\Omega\big(\sqrt{k}\big)$ for the clustering step, we show that one can accurately determine which cluster a new task belongs to with only $t_i = \Omega(\log k)$ examples once we have a rough initial estimation $\tilde\W$ of the parameter $\W$. 
We grow the clusters by adding tasks with %$t = \Omega(\log k)$ 
%\Raghav{$\log kn_{L2}$? $\tilde\Omega$?} 
a logarithmic number of examples until we have enough data points per cluster to achieve the desired  accuracy. We call this step \textit{classification}. This concludes our algorithm for the parameter estimation (i.e. meta-learning) phase.

\section{Main results}
\label{sec:main}

Suppose we have 
$n_H$ heavy tasks each with at least $t_H$ training examples, and $n_L$ light tasks
each with at least $t_L$ training examples. 
If heavy tasks are  data rich ($t_H\gg d$), we can learn $\W$ straightforwardly from a relatively small number, i.e.~$n_H = \Omega(k \log k)$.
If the light tasks are data rich ($t_L \gg k$), they can be straightforwardly clustered on the projected $k$-dimensional subspace.  
We therefore focus on the  following challenging regime of data scarcity.

\begin{asmp}
    \label{asmp}
    The heavy dataset ${\cal D}_{H}$ consists of 
    $n_H$ heavy tasks, each with at least $t_H$ samples. 
    The first light dataset ${\cal D}_{L1}$
    consists of $n_{L1}$ light tasks, 
    each with at least $t_{L1}$ samples. 
    The second light dataset ${\cal D}_
    {L2}$
    consists of 
    $n_{L2}$ tasks,
    each with at least $t_{L2}$ samples.
    We assume $t_{L1}, t_{L2}<k, $ and $t_H<d$. 
\end{asmp}

To give  more fine grained analyses on the 
sufficient conditions, 
we assume two types of light tasks are available with potentially differing sizes (Remark~\ref{rem:twotype}). 
In meta-learning step in Algorithm~\ref{alg:estimate},  subspace estimation uses ${\cal D}_{L1}$, 
clustering uses ${\cal D}_H$, and 
classification uses ${\cal D}_{L2}$. 
We provide proofs of the main results  in Appendices~\ref{sec:proof_est},  \ref{sec:proof_pred}, and 
\ref{sec:lb-predict-proof}.

%{\bf implicitly assume $s_i=\BigO{1}$ and keep track of only $\Delta$.} 

% =============================
\subsection{Meta-learning}

We characterize a sufficient condition to
achieve a target accuracy $\epsilon$ in estimating the meta-parameters $\theta=(\W,\s,\p)$. 

\begin{thm}[Meta-learning]\label{thm:main_reg_est}
For any failure probability $\delta \in (0,1)$, and accuracy $\epsilon \in (0,1)$, given three batches of samples under Assumption~\ref{asmp}, meta-learning step of Algorithm~\ref{alg:estimate} estimates the meta-parameters with accuracy 
\begin{align*}
\enorm{\hat{\w}_i-\w_i} &\;\le\; \eps s_i\; , \\ 
\abs{\hat{s}_i^2-s_i^2}&\;\le\;  \frac{\eps}{\sqrt{d}}s_i^2\; ,\quad\text{and}\\
\abs{\hat{p}_i-p_i} &\;\le\; \eps \sqrt{\frac{t_{L2}}{d}}\;p_i\;, 
\end{align*}
with probability at least $1-\delta$, if the following holds. The numbers of tasks satisfy 
\begin{align*}
n_{L1} &= \Omega \round{\frac{d\,\log^3\round{\frac{d}{p_{\rm min}\Delta\delta}}}{t_{L1}} \cdot \min \curly{ \Delta^{-6} p_{\min}^{-2} , \Delta^{-2} \lambda_{\min}^{-2} }  }\;,\\
%\end{align*} 
%number of examples for \textit{subspace estimation}, 
%\begin{eqnarray*}
n_{H} & =\; \Omega\round{\frac{  \, \log(k/\delta)}{t_H\; p_{\rm min}\Delta^2}\round{k+\Delta^{-2}}}\;,\\
%\end{eqnarray*}
%i.i.d. tasks each with 
%\begin{eqnarray*}
n_{L2} & =\; \Omega\round{\frac{  \,d\log^2(k/\delta)}{t_{L2}p_{\rm min}\eps^2 }}\;,
\end{align*}
and the numbers of samples  per task  satisfy 
$t_{L1}\ge 2$, 
 % number of examples for \textit{clustering}, the third batch with 
 $t_{L2} = \Omega\round{\log\round{{kd}/({p_{\rm min}\delta\eps)}}/\Delta^4}$, 
 and $t_H = \Omega\round{\Delta^{-2}\sqrt{k}\log\round{{k}/({p_{\rm min}\Delta\delta})}}$, where $\lambda_{\min}$ is the smallest non-zero eigen value of $\M\coloneqq\sum_{j=1}^k p_j \w_j\w_j^\top \in \mathbb{R}^{d \times d}$. 
%examples for \textit{classification}, . 
\end{thm}

In the following remarks, we explain each of the conditions.  

\begin{remark}[Dependency in ${\cal D}
_{L1}$]
The total number of samples used in subspace estimation is  $n_{L1} t_{L1}$. The sufficient condition scales linearly in $d$ which matches the  information theoretically necessary condition up to logarithmic factors. 
If the  matrix $\M$ is well conditioned, for example when $\w_i$'s are all orthogonal to each other, subspace estimation is easy, and $n_{L1}t_{L1}$ scales as $\Delta^{-2} \lambda_{\min}^{-2}$. Otherwise, the problem gets harder, and we need $\Delta^{-6} p_{\min}^{-2}$ samples.
% \Raghav{The tighter requirement now looks like
% \begin{eqnarray*}
%     n_{L1} &\geq \tilde\Omega\round{\max\curly{\min\curly{\Delta^{-2}\lambda_{\rm min}^{-2},\Delta^{-6} p_{\rm min}^{-2}},dt_{L1}^{-1}\min\curly{\Delta^{-1}\lambda_{\min}^{-1},\Delta^{-3}p_{\rm min}^{-1}}}}
% \end{eqnarray*}
% if $\log(n_{L1}d)\cdot\max\curly{\Delta\lambda_{\rm min},\Delta^3p_{\rm min}}\gtrsim 1$, which in the orthonormal uniform case means $k\lesssim \log(n_{L1}d)$.}\Weihao{Where does this come from?}
Note that in this regime, tensor decomposition approaches often fails to provide any meaningful guarantee (see Table~\ref{tab:related}). In proving this result, 
we improve upon a matrix perturbation bound  in % Lemma~5 of 
\cite{li2018learning} to 
shave off a $k^6$ factor on $n_{L1}$
 (see Lemma~\ref{lem:lemma_5_in_ll18}). 
 
% In an ideal case where, $p_{\rm min}=\Omega(1/k)$, $\Delta=\Omega(1)$, and the $\w_i$'s are well separated, the sufficient condition for subspace estimation is  $n_{L_1}t_{L_1}=\BigO{k^2 d}$. This is a factor of $k$ larger than the information theoretic necessary condition, ...???
 
\end{remark}

\begin{remark}[Dependency in ${\cal D}_H$]
The 
clustering step 
requires 
$t_H =\tilde{\Omega}(\sqrt{k})$, which is necessary for distance-based 
  clustering approaches such as 
  single-linkage clustering. 
  From \cite{kong2018estimating, kong2019sublinear}   we know that it is necessary (and sufficient) 
  to have $t = \Theta(\sqrt{k})$, 
even for a  simpler 
 testing problem between $\beta_1 = \beta_2$ or $\esqnorm{\beta_1-\beta_2}\gg 0$, from 
 two labelled datasets  
 %of data points $\{(\x_{1,i}, y_{1,i})\in{\mathbb R}^{k+1}\}_{i\in [t]}$ and $\{(\x_{2,i}, y_{2,i})\in{\mathbb R}^{k+1}\}_{i\in [t]}$ whose true (unknown) regression vectors are 
 with two linear models $\beta_1$ and $\beta_2$. 
 
 Our clustering step is inspired by \cite{vempala2004spectral} on clustering under Gaussian mixture models, where  
 the algorithm succeeds if  
 $t_H =\tilde\Omega( \Delta^{-2} \sqrt{k} )$. 
 Although a straightforward adaptation  fails, 
 we match the sufficient condition. 
% The key is using the median of estimates to get   Lemma~\ref{lem:distance}. 
 
We only require the  number of heavy samples $n_Ht_H$ to be $\tilde\Omega\round{{k}/{p_{\rm min}}}$  up to logarithmic factors,  
 which is information theoretically necessary.
 %such that the smallest cluster has $n_2t_2p_{\rm min} = \Omega(k)$ samples. This matches the information theoretic lower bound of a  $k$-dimensional regression problem. 
\end{remark}

\begin{remark}[Gain of using two types of light tasks]
\label{rem:twotype}
To get the tightest guarantee, it is necessary to use a different set of light tasks to perform the final estimation step. 
First notice that the first light dataset $\Dcal_{L1}$ does not cover the second light dataset since we need $t_{L2}\ge \Omega(\log (kd))$ which does not need to hold for the first dataset $\Dcal_{L1}$. On the other hand, the second light dataset does not cover the first light dataset in the setting where $\Delta$ or $p_{\rm min}$ is very small. 

%If we use the (previously clustered) heavy tasks, we would need the number of heavy tasks $n_H$ to be at least $\Omega(d/\eps^2)$, which is unnecessarily large.
%If we use the (previously used) light tasks, we would need 
\end{remark}

\begin{remark}[Dependency in ${\cal D}_{L2}$]
Classification and prediction use the same routine to classify the given task. 
 Hence, the $\log k$ requirement in $t_{L2}$ is tight, as it matches  
 our lower bound in  Proposition~\ref{thm:lb-predict}. 
 The extra terms in the $log$ factor come from the union bound over all $n_{L2}$ tasks to make sure all the tasks are correctly classified. It is possible to replace it by $\log(1/\eps)$ by showing that $\eps$ fraction of incorrectly classified tasks does not change the estimation by more than $\eps$. We only require $n_{L2}t_{L2} = \Omega(d/p_{\rm min})$ up to logarithmic factors, 
%ensures that the smallest cluster has at least $d$ samples, 
which is information  theoretically necessary. 
%such that the smallest cluster has $n_3t_3p_{\rm min} = \Omega(d)$ samples since $\Omega(d)$ samples is information theoretically necessary for solving the $d$ dimensional regression problem. 
%The reason for not performing the classification step in the estimated $k$-dimensional subspace  is  
\end{remark}

% =============================
\subsection{Prediction}

Given an estimated meta-parameter 
$\hat\theta=(\hat\W,\hat\s,\hat\p)$, and a new dataset ${\cal D}=\{(\x^{\rm new}_j,y^{\rm new}_j)\}_{j\in[\tau]}$, 
we make predictions on the new task with unknown parameters using two estimators: MAP estimator and Bayes optimal estimator.
%Given a task $\phi^{\rm new} \sim {\mathbb P}_{\theta}$ and its dataset $\cal D \sim {\mathbb P}_{\phi^{\rm new}}$, 

\begin{definition}
\label{def:predictor}
Define the maximum a posterior (\textit{MAP}) estimator as
\[
\hat{\beta}_{\rm MAP}({\cal D}) \coloneqq \hat{\w}_{\hat{i}}\; ,\quad\text{where}\quad \hat{i} \coloneqq \argmax_{i\in\squarebrack{k}} \log \hat{L}_i\;, \text{ and }
\]
\begin{eqnarray*}
\hat{L}_{i} &\coloneqq \exp \round{-\sum\limits_{j=1}^\tau\frac{\round{y^{\rm new}_j-\hat\w_i^\top \x_j^{\rm new}}^2}{2\hat{s}_i^2}-\tau\log\hat{s}_i+\log\hat{p}_i}.
\end{eqnarray*}
Define the posterior mean estimator as
\[
\hat{\beta}_{\rm Bayes}({\cal D}) \coloneqq \frac{\sum_{i=1}^k \hat{L}_i\hat{\w}_{i}}{\sum_{i=1}^k \hat{L}_i}.
\]
\end{definition}

If the true prior, $\{(\w_i,  s_i, p_i)\}_{i\in [k]}$, is known. The posterior mean estimator achieves the smallest expected squared $\ell_2$ error, ${\mathbb E}_{{\cal D},\beta^{\rm new}}\squarebrack{{\esqnorm{{\hat{\beta}({\cal D})-\beta^{\rm new}}}}}$. Hence, we refer to it as 
%In this paper, we also call the posterior mean estimator 
Bayes optimal estimator. The MAP estimator maximizes the probability of exact recovery. 

\begin{thm}[Prediction]\label{thm:main_reg_pred}
Under the hypotheses of Theorem~\ref{thm:main_reg_est} with $\eps\le \min\curly{\Delta/10, \Delta^2\sqrt{d}/50}$, %given a new task $\cal T^{\rm new}$ with 
the expected prediction errors of both the MAP  and Bayes optimal estimators $\hat\beta({\cal D})$ are bound as
\begin{align}
&
%\myE_{\cal T^{\rm new} \sim \P{(\cal T)}}
%\myE_{\cal D \sim {\mathbb P}_{\phi^{\rm new}} }\myE_{\{\x,y\}\sim \cal T^{\rm new}}
{\mathbb E}\squarebrack{\round{\x^\top \hat{\beta}({\cal D})-y}^2}
\;\le\;  \delta+\round{1+\eps^2}\sum_{i=1}^kp_is_i^2\; ,
\label{eq:pred}
\end{align}
%and the same bound for the Bayes optimal estimator holds
%\begin{align*}
%&\mathbb{E}_{\cal T^{\rm new} \sim \P{(\cal T)}}\mathbb{E}_{\cal D \sim \cal T^{\rm new}}\mathbb{E}_{\{\x,y\}\sim \cal T^{\rm new}}\squarebrack{\round{\x^\top \hat{\beta}_{\rm Bayes}({\cal D})-y}^2}\\
%\le&\;  \delta+(1+\eps^2)\sum_{i=1}^kp_is_i^2\; .
%\end{align*}
if $\tau\ge \Theta\round{\log (k/\delta)/\Delta^4}$, 
where
the true meta-parameter is 
$\theta=\{(\w_i,s_i,p_i)\}_{i=1}^k$, 
the expectation is over the new task with model parameter $\phi^{\rm new}=(\beta^{\rm new},\sigma^{\rm new}) \sim {\mathbb P}_{\theta} $, training dataset ${\cal D}\sim{\mathbb P}_{\phi^{\rm new}} $, 
and test data $(\x,y)\sim {\mathbb P}_{\phi^{\rm new}}$. 
\end{thm}

%\begin{remark}
%%Denote the true regression vector for $\cal T^{\rm new}$ as $\beta^{\rm new}$. The MAP estimator estimate $\w_h$ with 
%The expected squared $\ell_2$ error is bounded as
%\begin{eqnarray*}
%%\mathop{\mathbb{E}}_{\cal T^{\rm new} \sim \P{(\cal T)}}\mathop{\mathbb{E}}_{\cal D \sim \cal T^{\rm new}}
%{\mathbb E} 
%\squarebrack{\esqnorm{\hat{\beta}({\cal D})-\beta^{\rm new}}} \; \le \;  \delta+\eps^2\sum\limits_{i=1}^kp_is_i^2,
%\end{eqnarray*}
%for both  Bayes optimal  and MAP estimators $\hat\beta({\cal D})$, 
%if $\tau\ge \Theta\round{\log (k/\delta)/\Delta^4}$. 
%%\begin{eqnarray*}
%%\mathop{\mathbb{E}}_{\cal T^{\rm new} \sim \P{(\cal T)}}\mathop{\mathbb{E}}_{\cal D \sim \cal T^{\rm new}}\squarebrack{\esqnorm{\hat{\beta}_{\rm Bayes}({\cal D})-\beta^{\rm new}}} \le \delta+\eps^2\sum\limits_{i=1}^kp_is_i^2.
%%\end{eqnarray*}
%\end{remark}

Note that the $\sum_{i=1}^k p_is_i^2$ term in \eqref{eq:pred} is due to the noise in $y$, and can not be avoided by any estimator. 
With an accurate meta-learning, we can achieve a prediction error arbitrarily close to this statistical limit, with $\tau=\BigO{\log k}$. 
Although both predictors achieve the same guarantee, Bayes optimal estimator achieves smaller training and test errors in Figure~\ref{fig:pred}, especially in challenging regimes with small data.
%parameter estimation error as expected. Further, the prediction error is also smaller. 

\begin{figure}[ht]
    \vskip -0.1in
    \begin{center}
    \subfigure[Training error]{\label{fig:train_err_32}\includegraphics[width=0.48\linewidth]{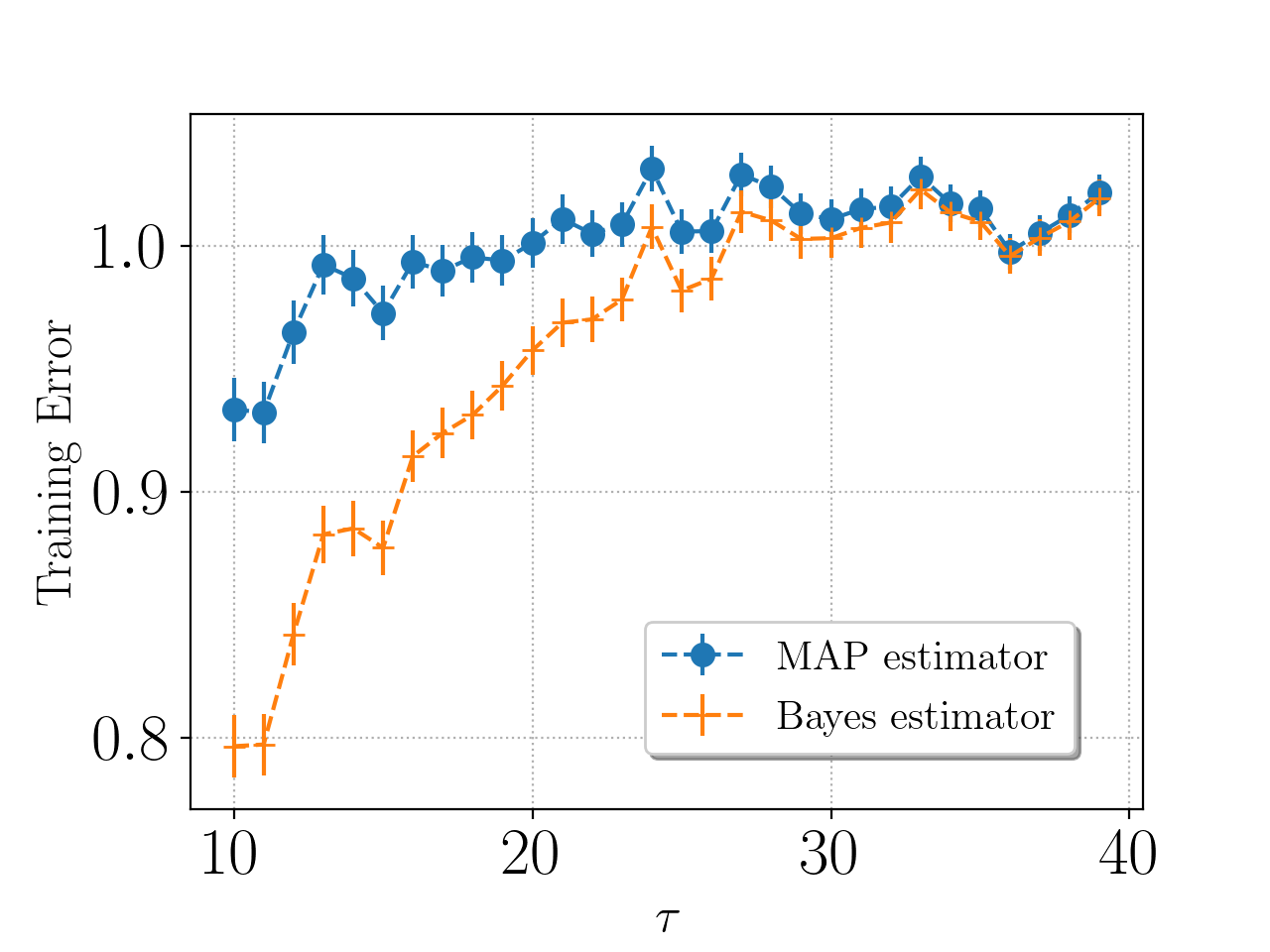}}
    \subfigure[Prediction error]{\label{fig:pred_err_32}\includegraphics[width=0.48\linewidth]{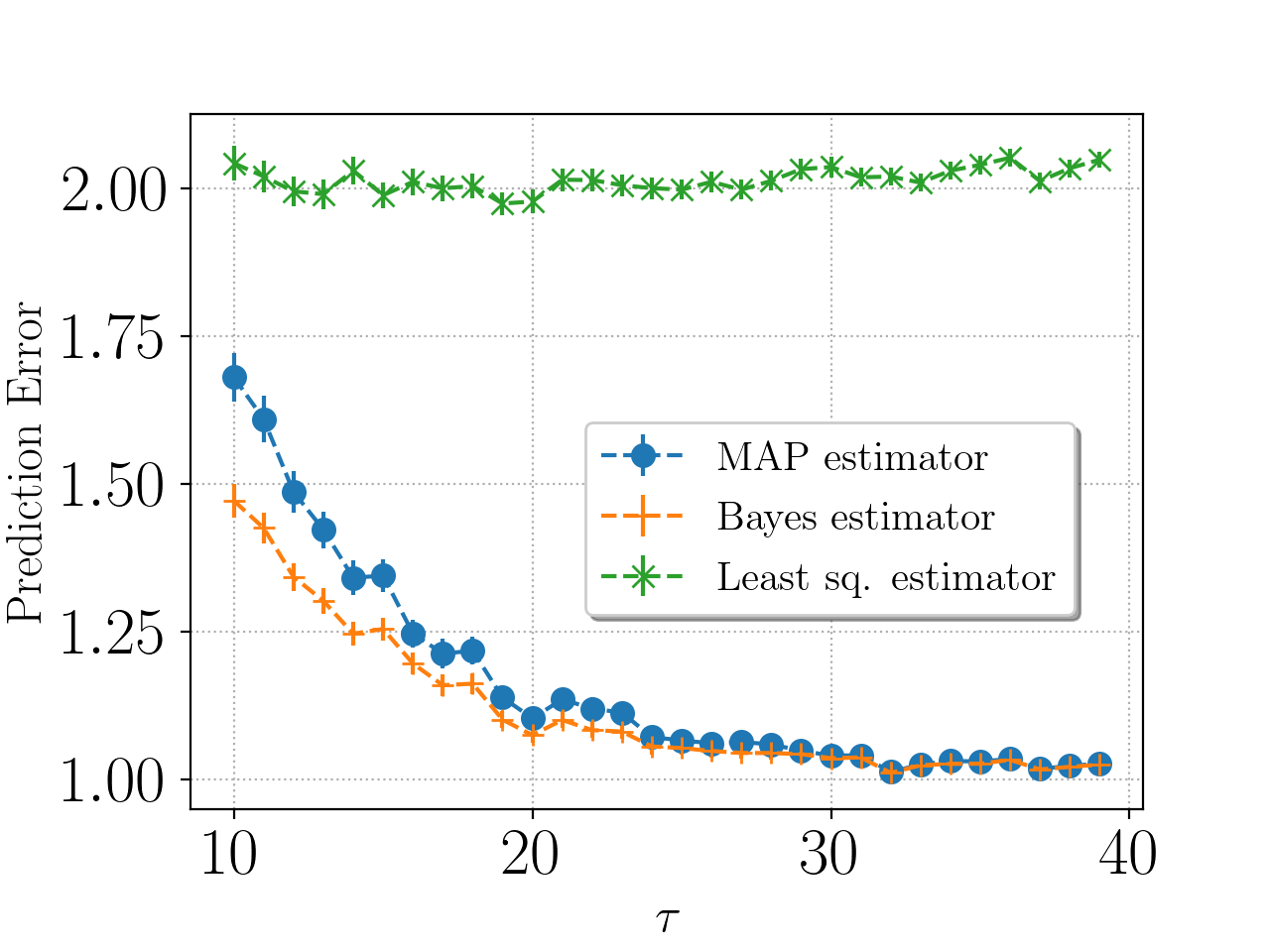}}
    \caption{Bayes optimal estimator achieves smaller errors for an example. Here, $k=32$, $d=256$, $\W^\top\W = \I_k$, $\s=\one_k$, $\p = \one_k/k$, and $\Pcal_\x$ and $\Pcal_\eps$ are standard Gaussian distributions. The parameters were learnt using the Meta-learning part of Algorithm~\ref{alg:estimate} as a continuation of simulations discussed in Appendix~\ref{sec:simulations}, where we provide extensive  experiments confirming our analyses.}
    \label{fig:pred}
    \end{center}
 %   \vskip -0.2in
\end{figure}

We show that $\tau= \Omega(\log k)$ training samples are necessary (even if the ground truths meta-parameter $\theta$ is known) to achieve error approaching this statistical limit. Let $ \Theta_{k,\Delta,\sigma}$ denote the set of all meta-parameters with $k$ components, satisfying $\|\w_i-\w_j\|_2 \geq \Delta$ for $i\neq j\in[k]$ and $s_i\leq \sigma$ for all $i\in [k]$. 
The following minimax lower bound shows that 
there exists a threshold scaling as $\BigO{\log k}$ below which no algorithm can achieve the fundamental limit of  $\sigma^2$, which is $\sum_{i=1}^k p_is_i^2$ in this minimax setting.  

\begin{remark}[Lower bound for prediction]\label{thm:lb-predict}
For any $\sigma,\Delta>0$, 
if $\tau =  \round{(1+\Delta^2)/\sigma^2}^{-1}\log(k-1)$, then 
    \begin{eqnarray}
    \inf_{\hat{y}}
    \sup_{\theta\in\Theta_{k,\Delta,\sigma}}
    \E{ \round{\hat{y}({\cal D},\theta)- y}^2} \; = \; \sigma^2 + \Omega\round{\Delta^2} \;,
    \end{eqnarray}
     where the minimization is over all measurable functions of the meta-parameter $\theta$ and the training data ${\cal D}$ of size $\tau$.
    \label{thm:lb}
\end{remark}

\section{Details of the algorithm and the analyses}
\label{sec:algorithm_detail}

We explain and analyze  each step in  Algorithm \ref{alg:estimate}. 
These analyses  imply our main result in meta-learning, which is explicitly written in Appendix~\ref{sec:proof_est}.

\subsection{Subspace estimation} 

In the following, we use  $k$\_SVD$(\cdot, k)$ routine that outputs the top $k$-singular vectors.
As ${\mathbb E}[{\hat\M}] = 
\M :=\sum_{j=1}^k p_j \w_j\w_j^\top$, 
this outputs an estimate of the subspace spanned by the true parameters. 
We show that as long as $t_{L1}\geq 2$, the accuracy only depends on the total 
number of examples, and it is sufficient to have 
$n_{L1}t_{L1}=\tilde{\Omega}(d)$.  

\begin{algorithm}[h]
   \caption{Subspace estimation}
   \label{alg:subspace_estimation}
\begin{algorithmic}
   \STATE {\bfseries Input:} data $\Dcal_{L1} = \curly{(\x_{i,j},y_{i,j})}_{i\in[n_{L1}],j\in[t_{L1}]}$, $k\in\Natural$
%   \REPEAT
%   \STATE Initialize $noChange = true$.
   \STATE {\bf compute} for all $i\in[n_{L1}]$ 
   %\FOR{$i=1$ {\bfseries to} $n_{L1}$}
   \STATE \hspace{-.5cm} $\hat{\beta}^{(1)}_i \gets \frac{2}{t_{L1}}\sum\limits_{j=1}^{t_{L1}/2}y_{i,j}\x_{i,j}$,\ $\hat{\beta}^{(2)}_i \gets \frac{2}{t_{L1}}\sum\limits_{j=t_{L1}/2+1}^{t_{L1}}y_{i,j}\x_{i,j}$
%   \IF{$x_i > x_{i+1}$}
%   \STATE Swap $x_i$ and $x_{i+1}$
%   \STATE $noChange = false$
%   \ENDIF
   %\ENDFOR
   \STATE $\hat\M \gets \round{2n_{L1}}^{-1}\sum_{i=1}^{n_{L1}} \round{\hat\beta_i^{(1)}\hat\beta_i^{(2)\top} + \hat\beta_i^{(2)}\hat\beta_i^{(1)\top}}$
   \STATE $\U \gets $ $k$\_SVD$\round{\hat\M,k}$
%   \UNTIL{$noChange$ is $true$}
    \OUTPUT $\U$
\end{algorithmic}
\end{algorithm}

The dependency on the accuracy $\eps$ changes based on the ground truths meta-parameters. 
In an ideal case when $\W$ is an orthonormal matrix (with condition number one), the sample complexity is $\tilde{\Ocal}\round{d/(p_{\rm min}^2 \varepsilon^2) }$. 
For the worst case $\W$, 
it is $\tilde{\Ocal}\round{d/\round{p_{\rm min}^2\varepsilon^6}}$. 
% The standard deviation  $\rho$ of the labels captures the scaling of the problem parameters, and we normalize the error by $\rho$. 
%\Weihao{There is something wrong in our proof. We showed concentration the $\beta$ matrix instead of $w$ matrix, and $\lambda_{min}$ should be for the $\beta$ matrix as well. How are the two $\lambda_{min}$ related?}\Raghav{Lemma~\ref{lem:bound_X_Xbar}?}

\begin{lemma}[Learning the subspace]
\label{lem:projection}
Suppose Assumption~\ref{asmp} holds, and let $\U\in \Rd{d\times k}$ be the matrix with top $k$ eigen vectors of matrix $\hat\M\in{\mathbb R}^{d\times d}$. For any failure probability   $\delta \in (0,1)$ and accuracy $\eps \in (0,1)$, if the sample size is large enough such that 
\begin{align*}
%n_{L1} &= \Omega \round{  d t_{L1}^{-1}\cdot \min \curly{ \eps^{-6} p_{\min}^{-2} , \eps^{-2} \lambda_{\min}^{-2} \rho^4 } \cdot \log^3(nd/\delta)},
n_{L1} &= \Omega \round{  d t_{L1}^{-1}\cdot \min \curly{ \eps^{-6} p_{\min}^{-2} , \eps^{-2} \lambda_{\min}^{-2}  } \cdot \log^3(nd/\delta)},
\end{align*} 
and $2 \leq t_{L1} <d$, we have 
\begin{align}\label{eq:projection}
%\frac{1}{\rho}
\enorm{(\U\U^\top-\I) \w_i} \;\; \le \;\; \eps\quad,
\end{align}
for all $i\in[k]$ 
with probability at least $1-\delta$, where $\lambda_{\rm min}$ is the smallest non-zero eigen value of $\M:=\sum_{j=1}^k p_j \w_j\w_j^\top$. 
%and $\rho^2=\max_{i\in[k]} s_i^2 + \|\w_i\|_2^2$ is the maximum variance of the label. 
\end{lemma}

{\textbf{Time complexity:}} 
%Note that $\hat\M$ can be computed as
% First we analyze the time complexity for constructing the estimator $\hat{\M}$. To speed up the computation, we can use matrix multiplication and compute $\hat{\M}$ with 
%\begin{align*}
%\B^{(1)}&\gets \begin{bmatrix}\hat{\beta}^{(1)}_1&\ldots&\hat{\beta}^{(1)}_{n_{L1}}\end{bmatrix},\; \B^{(2)} \gets \begin{bmatrix}\hat{\beta}^{(2)}_1&\ldots&\hat{\beta}^{(2)}_{n_{L1}}\end{bmatrix}\\
%\hat\M &\gets \round{2n_{L1}}^{-1}\B^{(1)}\B^{(2)\top}+\B^{(2)}\B^{(1)\top}.
%\end{align*}
 $\BigO{\round{n_{L1}^{\omega-1}+n_{L1}t_{L1}}d}$ 
  for 
 computing $\hat\M$, and $\BigO{kd^2}$  for $k$\_SVD 
\cite{AL16}.

\subsection{Clustering}

Once we have the subspace, we can efficiently cluster any 
 task associated with  $t_H=\tilde\Omega(\sqrt{k})$ samples. 
In the following, the matrix $\Hmat \in{\mathbb R}^{n_H\times n_H}$ estimates 
the distance between the parameters in the projected $k$-dimensional space. 
If there is no error in $\U$, then ${\mathbb E}[\Hmat_{i,j}]\geq \Omega\round{\Delta^2}$ if $i$ and $j$ are from different components, and zero otherwise. 
Any clustering algorithm can be applied treating $\Hmat$ as a distance matrix.
% Our analysis holds for a specific choice of  \Sewoong{explain our clustering, is it single linkage?? or just thresholding?? }

%any non-degenerate clustering algorithm that returns the correct clusters when applied to a graph with $k$ disjoint cliques. 

% which can be any algorithm which takes in a pair-wise distance matrix and outputs the cluster.

\begin{algorithm}[h]
   \caption{Clustering and estimation}
   \label{alg:clustering}
\begin{algorithmic}
   \STATE {\bfseries Input:} data $\Dcal_H = \curly{(\x_{i,j},y_{i,j})}_{i\in[n_H],j\in[t_H]}$, $2L\leq t_H$, $k\in\Natural$, $L\in\Natural$, $\U\in\Rd{d\times k}$
   \STATE {\bf compute} for all $\ell\in[L]$ and $i\in[n_H]$
    \STATE \hspace{0.3cm}  $\beta^{(\ell)}_i \gets \round{\nicefrac{2L}{t_H}}\sum_{j=(\ell-1)\cdot\round{\nicefrac{t_H}{2L}}+1}^{\ell\cdot\round{\nicefrac{t_H}{2L}}}y_{i,j}\x_{i,j}$
        \STATE \hspace{0.3cm}  $\beta^{(\ell+L)}_i \gets \round{\nicefrac{2L}{t_H}}\sum_{j=\ell \cdot\round{\nicefrac{t_H}{2L}}+1}^{2\ell \cdot\round{\nicefrac{t_H}{2L}}}y_{i,j}\x_{i,j}$
   \STATE {\bf compute} for all $\ell\in[L]$ and $(i,j)\in[n_H]\times [n_H]$
%   \FOR{$l=1$ {\bfseries to} $L$}
 %       \FOR{$(i,j)\in\squarebrack{n_2}\times\squarebrack{n_2}$}
            \STATE \hspace{0.3cm} $\Hmat_{i,j}^{(\ell)} \gets \round{\hat\beta_i^{(\ell)} - \hat\beta_j^{(\ell)}}^\top\U \U^\top\round{\hat\beta_i^{(\ell+L)} - \hat\beta_j^{(\ell+L)}}$
            %\Sewoong{$\ell'$ is not defined, acn you change all $l$ to $\ell$?}
%        \ENDFOR
%   \ENDFOR
%   \FOR{$(i,j)\in\squarebrack{n_2}\times\squarebrack{n_2}$}
    \STATE {\bf compute} for all $(i,j)\in[n_H] \times [n_H]$
        \STATE \hspace{0.3cm} $\Hmat_{i,j} \gets $ median$\round{ \{\Hmat_{i,j}^{(\ell)}\}_{\ell\in[L]}}$
%    \ENDFOR
   \STATE Cluster $\Dcal_H$ using $\Hmat$ and return its partition $\curly{\Ccal_\ell}_{\ell\in[k]}$
    %\FOR{$l=1$ {\bfseries to} $k$}
    \STATE {\bf compute} for all $\ell\in[L]$
         \STATE \hspace{0.3cm} $\tilde\w_\ell \gets \round{t_H \abs{\Ccal_\ell}}^{-1}
         \sum_{i\in\Ccal_\ell,j\in[t_H]}\,
         y_{i,j}\U\U^\top\x_{i,j}$
         \STATE \hspace{0.3cm} $\tilde{r}_\ell^2 \gets \round{t_H\abs{\Ccal_\ell}}^{-1}
         \sum_{i\in\Ccal_\ell,j\in[t_H]}\,
         \round{y_{i,j}-\x_{i,j}^\top\tilde\w_\ell}^2$
        \STATE \hspace{0.3cm} $\tilde{p}_\ell \gets \abs{\Ccal_\ell}/n_H$
    %\ENDFOR
    \OUTPUT $\curly{\Ccal_\ell,\; \tilde\w_\ell,\; \tilde{r}^2_\ell,\; \tilde{p}_\ell}_{\ell=1}^k$
\end{algorithmic}
\end{algorithm}

This is inspired by \cite{vempala2004spectral}, where 
clustering mixture of Gaussians is studied.
One might wonder if it is possible to apply their clustering approach to $\hat{\beta}_i$'s directly. This approach fails as it crucially relies on the fact that 
%with for $\x \sim \cal N(\mu, \I)$, it holds that 
$\enorm{\x-\mu} = \sqrt{k}\pm\tilde\Ocal(1)$ with high probability for $\x\sim \N(\zero,\I_k)$.
Under our linear regression setting, $\enorm{y\x-\beta}$ does not concentrate. 
We instead propose median of estimates, to get the desired $t_H=\tilde\Omega(\sqrt{k})$ sufficient condition. 

%\begin{figure}[ht]
%    \vskip 0.2in
%    \begin{center}
%    \subfigure[Without median of estimates]{\label{fig:hist_overlap}\includegraphics[width=.49\linewidth]{images/hist_mean.png}}
%    \subfigure[With median of estimates]{\label{fig:hist_separate}\includegraphics[width=.49\linewidth]{images/hist_median.png}}
%    \caption{Histogram of $\curly{\Hmat_{i,j}}_{i,j\in\squarebrack{n}\times\squarebrack{n}}$ illustrates that median improves the concentration of the distance estimates. \Raghav{Had to use poly-tail distributions like $\x_{i,j},\eps_{i,j}\sim p(x)\propto\inv{\abs{x}^4}$ to observe better concentration of median for $d=k=3$.}}
%    \label{fig:median}
%    \end{center}
%    \vskip -0.2in
%\end{figure}

\begin{lemma}[Clustering and initial parameter estimation]
\label{lem:distance}
Under Assumption~\ref{asmp}, and given an orthonormal matrix $\U\in{\mathbb R}^{d\times k}$ satisfying 
\eqref{eq:projection} with any  %$\eps\in\round{0,{\nicefrac{\Delta}{4\rho}}}$,
$\eps\in\round{0,{\nicefrac{\Delta}{4}}}$, Algorithm~\ref{alg:clustering} correctly clusters all tasks  with 
%$t_H = \Omega\round{\rho^2\Delta^{-2}\sqrt{k}\log({n}/{\delta})}=\BigO{k}$ 
$t_H = \Omega( \Delta^{-2}\sqrt{k}\log({n}/{\delta}))$ 
 with probability at least $1-\delta$, $\ \forall\ \delta\in\round{0,1}$. 
Further, if
% and $\tilde\eps\in\round{0,\sqrt{p_{\rm min}k/t}}$ and if 
\begin{eqnarray}
%n \;\; =\;\; \Omega\round{\frac{\Delta^2 \, \sqrt{k}}{p_{\rm min}(\rho\tilde\eps)^2\,\log\round{\nicefrac{k}{\eps\delta}}}}\;,
n_H \;\; =\;\; \Omega\round{\frac{  k\, \log(k/\delta)}{t_H\,\tilde\eps^2\,p_{\rm min}}}\;, 
\end{eqnarray}
  for any $\tilde\eps>0$, with probability at least $1-\delta$, 
  \begin{subequations}
    \begin{eqnarray}
    %\abs{\hat{p}_i-p_i}&\leq& \tilde\eps\sqrt{\frac{p_{\rm min}t}{k}}\\
    %\esqnorm{\U^\top(\hat\w_i-\w_i)} &\leq& \tilde\eps\rho\\
    \esqnorm{\U^\top(\tilde\w_i-\w_i)} &\leq& \tilde\eps \\
    \abs{\tilde{r}_i^2 - r_i^2} &\leq& \frac{\tilde\eps}{\sqrt{k}}r_i^2 \;,
  \end{eqnarray}
  \end{subequations}
where $r_i^2 \coloneqq (s_i^2 + \esqnorm{\tilde\w_i-\w_i})$ for all $ i\in\squarebrack{k}$.
%and  $\hat{\w}_i$ and $\hat{r}_i^2$ are computed according to Algorithm~\ref{alg:clustering}.
\end{lemma}
{\textbf{Time complexity:}} It takes $\BigO{n_Hdt_H+n_Hdk}$ time to compute $\{\U^\top\hat{\beta}_i^{(l)}\}_{i\in\squarebrack{n_H},l\in\squarebrack{L}}$. Then by using matrix multiplication, it takes $\BigO{n_H^2 k^{\omega-2}}$ time to compute the matrix $\Hmat$, and the single linkage clustering algorithm takes $\BigO{n_H^2}$ time~\cite{sibson1973slink}.

\subsection{Classification}

Once we have $\curly{\tilde\w_\ell}_{\ell=1}^k$ from the clustering step, we can efficiently classify any task with $t_{L2}=\tilde\Omega(\log k)$ samples, and an extra $\log n_{L2}$ samples are necessary to apply the union bound. 
This allows us to use the light samples, in order to refine the  clusters estimated with heavy samples. This separation allows us to achieve the desired sample complexity on light tasks $(t_{L2}=\Omega(\Delta^{-4}\log d),n_{L2}t_{L2}p_{\rm min}=\tilde\Omega(\epsilon^{-2}d))$, and heavy tasks $(t_H=\tilde\Omega(\Delta^{-2}\sqrt{k}),n_Ht_Hp_{\rm min}=\tilde\Omega(\Delta^{-2} k))$. %Without the separation, we would require one batch of samples with  $t_{L2}=\tilde\Omega(\Delta^{-2}\sqrt{k})$, and $n_{L2}t_{L2}p_{\rm min}=\tilde\Omega(\epsilon^{-2} d)$.

In the following, we use Least\_Squares$(\cdot)$ routine that outputs the least-squares estimate of all the examples in each cluster. Once each cluster has $\BigO{d}$ samples, we can accurately estimate the meta-parameters.

\begin{algorithm}[h]
   \caption{Classification and estimation}
   \label{alg:classification}
\begin{algorithmic}
   \STATE {\bfseries Input:} data $\Dcal_{L2} =  \curly{(\x_{i,j},y_{i,j})}_{i\in[n_{L2}],j\in[t_{L2}]}$, $\curly{\Ccal_\ell,\; \tilde\w_\ell,\; \tilde{r}^2_\ell }_{\ell\in[k]}$
   \STATE {\bf compute} for all $i\in[n_{L2}]$ 
    $$h_i \gets \argmin_{\ell \in\squarebrack{k}} 
    \inv{2\tilde{r}_\ell^2}\sum_{j\in\squarebrack{t_{L2}}} \round{y_{i,j}-\x_{i,j}^\top\tilde\w_\ell}^2 + t_{L2} \log\tilde{r}_\ell$$
%   \FOR{$i=1$ {\bfseries to} $n_3$}
%        \FOR{$l=1$ {\bfseries to} $k$}
%      \STATE \hspace{0.3cm} $\hat{l}_i \gets - \inv{2\hat{r}_l^2}\sum_{j\in\squarebrack{t_3}} \round{y_{i,j}-\x_{i,j}^\top\hat\w_l}^2 - t_{L2} \log\hat{r}_l$
%        \ENDFOR
        \STATE \hspace{0.3cm}  $\Ccal_{h_i}\gets \Ccal_{h_i}\cup\curly{(\x_{i,j},y_{i,j})}_{j=1}^{t_{L2}}$
%   \ENDFOR
    %\FOR{$l=1$ {\bfseries to} $k$}
    \STATE {\bf compute} for all $\ell\in[k]$, 
        \STATE \hspace{0.3cm} $\hat\w_\ell \gets$ Least\_Squares$(\Ccal_\ell)$
        \STATE \hspace{0.3cm} $\hat{s}_\ell^2 \gets \round{t_{L2} \abs{\Ccal_\ell}-d}^{-1}
        \sum_{i \in\Ccal_\ell,j\in[t_{L2}]}
        \round{y_{i,j}-\x_{i,j}^\top\hat\w_\ell}^2 $
        \STATE \hspace{0.3cm} $\hat{p}_\ell \gets \abs{\Ccal_\ell}/n_{L2}$
    %\ENDFOR
    \OUTPUT $\curly{\Ccal_\ell,\; \hat\w_\ell,\; \hat{s}^2_\ell,\; \hat{p}_\ell}_{\ell=1}^k$
\end{algorithmic}
\end{algorithm}

\begin{lemma}[Refined parameter estimation via classification]\label{lem:refined-est}
Under Assumption~\ref{asmp} and given estimated parameters $\tilde{\w}_i$, $\tilde{r}_i$ satisfying $\enorm{\tilde{\w}_i-\w_i} \le \Delta/10$,  %$\round{1-\Delta^2/50\rho^2}\hat{r}_i^2\le s_i^2+\esqnorm{\hat{\w}'_i-\w_i}\le \round{1+\Delta^2/50\rho^2}\hat{r}_i^2$ for all $i\in [k]$ and 
$\round{1-\Delta^2/50}\tilde{r}_i^2\le s_i^2+\esqnorm{\tilde{\w}_i-\w_i}\le \round{1+\Delta^2/50}\tilde{r}_i^2$ for all $i\in [k]$ and 
$n_{L2}$ task with $t_{L2} = \Omega\round{\log (k n_{L2}/\delta)/\Delta^4}$ examples per task, with probability $1-\delta$, Algorithm~\ref{alg:classification} correctly classifies all the $n_{L2}$ tasks. Further, for any $0 <\eps \le 1$ if 
\begin{eqnarray}
%n \;\; =\;\; \Omega\round{\frac{\Delta^2 \, \sqrt{k}}{p_{\rm min}(\rho\tilde\eps)^2\,\log\round{\nicefrac{k}{\eps\delta}}}}\;,
n_{L2} \;\; =\;\; \Omega\round{\frac{  \,d \log^2(k/\delta)}{t_{L2} p_{\rm min}\eps^2 }}\;, 
\end{eqnarray}
the following holds for all $i\in[k]$, 
\begin{subequations}
\label{eq:estimation_err}
\begin{eqnarray}
\enorm{\hat{\w}_i-\w_i} &\le& \eps s_i\; , \\ 
\abs{\hat{s}_i^2-s_i^2}&\le&  \frac{\eps}{\sqrt{d}}s_i^2\; ,\quad\text{and}\\
\abs{\hat{p}_i-p_i} &\le&  \eps \sqrt{t_{L2}/d}\,p_i.
\end{eqnarray}
\end{subequations}
\end{lemma}
{\textbf{Time complexity:}} Computing $\{h_i\}_{i\in\squarebrack{n_{L2}}}$ takes $\BigO{n_{L2}t_{L2}dk}$ time, and least square estimation takes $\BigO{n_{L2}t_{L2}d^{\omega-1}}$ time.

\section{Related Work}
\label{sec:related}

\begin{table*}[t]
\caption{Sample complexity for previous work in MLR to achieve small constant error on parameters recovery of the mixed linear regression problem. We ignore the constants and $\mathrm{poly}\log$ factors. Let $n,d$, and $k$ denote the number of samples, the dimension of the data points, and the number of clusters, respectively. 
\cite{yi2016solving} and~\cite{CL13} requires 
$\sigma_k$, the $k$-th singular value of some moment matrix. %and  %which is required by \cite{yi2016solving}. Let 
\cite{sedghi2016provable} requires $s_{\rm min}$, the $k$-th singular value of the matrix of the regression vectors.
 % in Sample Complexity ($\#$ Sample) and Running Time (Time). 
 Note that $1/s_{\rm min}$ and $1/\sigma_k$ can be infinite even when $\Delta>0$. \cite{zhong2016mixed} algorithm requires
 $\Delta_{\rm max}/\Delta_{\rm min} = \BigO{1}$ and
some spectral properties. } %Thus, our result is the first truly $\poly(k)$ time algorithm.
\label{tab:related}
\vskip 0.15in
\begin{center}
\begin{small}
\begin{sc}

\begin{tabular}{l c l} 
\toprule
{\bf References}  & {\bf Noise} & {\bf \# Samples $n$} \\
\midrule
\cite{CL13}  & Yes & $d^6 \cdot \poly(k, 1/\sigma_k)$  \\
\cite{yi2016solving}  & No & $d \cdot \poly(k,1/\Delta, 1/\sigma_k)$  \\
\cite{zhong2016mixed} & No  
%$\sigma_k(\M_2), \|\M_2\|,\|M_3\|_{\rm op}^{2/3}, \sum_{i\in[k]} \pi_i \|\w_i\|^2, (\sum_i \pi_i \|\w_i\|^{3})^{2/3}$ all have the same order in $d$
&$d \cdot \exp(k \log(k\log d)) $  \\
\cite{sedghi2016provable}  & Yes  & $d^3 \cdot \poly ( k, 1/s_{\rm min})$  \\
\cite{li2018learning}   & No & $d \cdot \poly(k/\Delta) + \exp(k^2 \log (k/\Delta) )$ \\
\cite{chen2019learning} & No  & $d \cdot \exp(\sqrt{k})\poly (1/\Delta)$  \\
%\cite{chen2019learning} & Yes & No & $d \cdot \exp(\sqrt{k}/\Delta^2)$ & $nd$ \\ \hline
%Ours & Yes & No & $d \cdot \poly(k)$ & $nd$ ? \\ \hline
\bottomrule
\end{tabular}

\end{sc}
\end{small}
\end{center}
\vskip -0.1in
\end{table*}

Meta-learning linear models 
have been studied in two contexts: mixed linear regression and multi-task learning.  

%Several algorithms with theoretical guarantees exist for multi-task learning, which addresses a more general setting. 

{\bf Mixed Linear Regression (MLR).}
When each task has only one sample, (i.e.~$t_i=1$), the problem has been widely studied.  
Prior work in MLR are summarized in Table~\ref{tab:related}. We emphasize that the sample  and time complexity of all the previous work either has a super polynomial dependency on $k$ (specifically at least $\exp(\sqrt{k})$) as in~\cite{zhong2016mixed, li2018learning, chen2019learning}), or depends on the inverse of the $k$-th singular value of some moment matrix as in~\cite{CL13, yi2016solving, sedghi2016provable}, which can be infinite.~\cite{chen2019learning} cannot achieve vanishing error when there is noise.

{\bf Multi-task learning.}  \cite{baxter2000model,ando2005framework,rish2008closed,orlitsky2005supervised} address a similar problem of finding 
%each parameter $\beta_i$ is assumed to lie on 
an unknown $k$-dimensional subspace, where all tasks can be accurately solved. 
The main difference is that 
all tasks have  the same  number   of examples, and the performance is evaluated on the  observed tasks used in training.  
Typical approaches use 
trace-norm to encourage low-rank solutions of the matrix $\begin{bmatrix}\hat\beta_i,\ldots,\hat\beta_n\end{bmatrix}\in{\mathbb R}^{d\times n}$. 
This is posed as a convex program~\cite{argyriou2008convex,harchaoui2012large,amit2007uncovering,pontil2013excess}.

Closer to our work is the streaming setting where $n$ tasks are arriving in an online fashion and one can choose how many examples to collect for each. \cite{balcan2015efficient} provides an online algorithm using a memory of size only $\BigO{k n + k d}$, but requires  
some tasks to have $t_i = \Omega\round{ dk/\epsilon^2}$ examples. 
In comparison, we only need  $t_H=\tilde\Omega(\sqrt{k})$ but use $\BigO{d^2+kn}$ memory.
\cite{bullins2019generalize} also use only small memory, 
but requires $\tilde\Omega\round{d^2}$ total  samples to perform the {\em  subspace estimation} under the setting studied in this paper.

%This allows  efficient methods of learning linear models to be seamlessly extended to non-linear models, and has been widely studied under the name of 

%Under the additive noise setting, existing  approaches require estimations and  operations of third order tensors, which require high sample and computational complexities.  \cite{CL13} constructs the tensor via tensor nuclear norm minimization, which requires $\BigO{\frac{d^6k^2}{\sigma_k^5\eps^2}}$ samples and $\BigO{n d^{12}}$ computations. \Zhao{Why not include the above result in the table. Sewoong: We were just separating out the cases where there is noise and the cases where there is no noise. We can decide how to present these results, when we know our final theorem.}
%\cite{sedghi2016provable} constructs the tensor using a score function, which only requires $\BigO{d^3k^4/(s^2\eps^2)}$ samples and $\BigO{nd^2}$ computations, where $s$ is the minimum singular value of the regression parameters $[\w_1, \cdots ,\w_k]$. % in the paper they say nd^2, but it is actually nkdL, with L number of initializations, which in not specified. 
%For a special case of $k=2$ components, \cite{CYC14} 

\noindent{\bf Empirical Bayes/Population of parameters.}
 A simple canonical setting of probabilistic meta-learning is when $\mathbb{P}_{\phi_i}$ is a univariate distribution (e.g. Gaussian, Bernoulli) and $\phi_i$ is the parameter of the distribution (e.g. Gaussian mean, success probability). 
 %which is drawn from a unknown prior $\mathbb{P}_\theta(\phi)$.
 Several related questions have been studied. In some cases, one might be interested in just learning the prior distribution $\mathbb{P}_\theta(\phi)$ or the set of $\phi_i$'s. For example, if we assume each student's score of one particular exam $x_i$ is a binomial random variable with mean $\phi_i$ (true score), given the scores of the students in a class, an ETS statistician~\cite{lord1969estimating} might want to learn the distribution of their true score $\phi_i$'s. Surprisingly, the minimax rate on estimating the prior distribution $\mathbb{P}_\theta(\phi)$ was not known until very recently~\cite{tian2017learning, vinayak2019maximum} even in the most basic setting where $\mathbb{P}_{\phi_i}(x)$ is  Binomial.

In some cases, similar to the goal of meta-learning, one might want to accurately estimate the parameter of the new task $\phi^{\rm new}$ given the new data $x^{\rm new}$, perhaps by leveraging an estimation of the prior $\mathbb{P}_\theta(\phi)$. This has been studied for decades under the \textit{empirical bayes} framework in statistics (see, e.g. the book by Efron~\cite{efron2012large} for an introduction of the field).

\section{Discussion}
\label{sec:discussion} 

We investigate how we can meta-learn when we have multiple tasks but each with a small number of labelled examples. This is also known as a few-shot supervised learning setting. When each task is a linear regression, we propose a novel spectral approach and show that we can leverage past experience on small data tasks to accurately learn the meta-parameters and predict new tasks. 

When each task is a logistic regression coming from a mixture model, then our algorithm can be applied seamlessly. 
However, the notion of  separation $\Delta = \min_{i\neq j} \|\w_i-\w_j\|_2$ does not capture the dependence on the statistical complexity. Identifying the appropriate notion of complexity 
 on the groundtruths meta-parameters is an interesting research question. 

The subspace estimation algorithm requires a total number of $\tilde\Omega(dk^2)$ examples. It is worth understanding whether this is also necessary. 
% information theoretically.

Handling the setting where $\Pcal_\x$ has different covariances in different tasks is a challenging problem. There does not seem to exist an unbiased estimator for $\W$. Nevertheless,~\cite{li2018learning} study the $t=1$ case in this setting and come up with an exponential time algorithm. Studying this general setting and coming up with a polynomial time algorithm for meta-learning in a data constrained setting is an interesting direction.

Our clustering algorithm requires the existence of medium data tasks with $t_H = \Omega(\sqrt{k})$ examples per task. It is worth investigating whether there exists a polynomial time and sample complexity algorithms that learns with $t_H = o(\sqrt{k})$. We conjecture that with the techniques developed in the robust clustering literature~\cite{diakonikolas2018list, hopkins2018mixture, kothari2018robust}, it is possible to learn with $t_H = o(\sqrt{k})$ in the expense of larger $n_H$, and higher computation complexity. For a lower bound perspective, it is worth understanding the information theoretic trade-off between $t_H$ and $n_H$ when $t_H=o(\sqrt{k})$.

\section{Acknowledgement}
Sham Kakade acknowledges funding from the Washington Research Foundation for Innovation in Data-intensive Discovery, and the NSF Awards CCF-1637360, CCF-1703574, and CCF-1740551.

\clearpage
\newpage
\bibliographystyle{icml2020}
\bibliography{ref}

\clearpage
\newpage
\onecolumn

\appendix
\section*{Appendix}
We provide proofs of main results and technical lemmas. 
\section{Proof of Theorem~\ref{thm:main_reg_est}} 
\label{sec:proof_est}

\begin{proof}[Proof of Theorem~\ref{thm:main_reg_est}]
First we invoke Lemma~\ref{lem:projection} with $\eps = \Delta/(10\rho)$ which outputs an orthonormal matrix $\U$ such that
\begin{align}
 \enorm{\round{\U\U^\top-\I} \w_i} \;\; \le \;\; \Delta/20
\end{align}
with probability $1-\delta$. This step requires a dataset with 
\begin{align*}
n_{L1} &= \Omega \round{  \frac{d}{t_{L1}} \cdot \min \curly{ \Delta^{-6} p_{\min}^{-2} , \Delta^{-2} \lambda_{\min}^{-2} } \cdot \log^3\round{\frac{d}{p_{\rm min}\Delta\delta}}}
\end{align*} 
i.i.d. tasks each with $t_{L1}$ number of examples.

Second we invoke Lemma~\ref{lem:distance} with the matrix $\U$ estimated in Lemma~\ref{lem:projection} and $\tilde\eps = \min\curly{\frac{\Delta}{20}, \frac{\Delta^2\sqrt{k}}{100}}$ which outputs parameters satisfying
\begin{eqnarray*}
    \enorm{\U^\top(\tilde\w_i-\w_i)}&\leq& \Delta/20\\
    \abs{\tilde{r}_i^2 - r_i^2} &\leq& \frac{\Delta^2}{100}r_i^2 \;.
\end{eqnarray*}
This step requires a dataset with 
\begin{eqnarray*}
n_{H} \;\; =\;\; \Omega\round{\frac{  \, \log(k/\delta)}{t_{H}\; p_{\rm min}\Delta^2}\round{k+\Delta^{-2}}}\;
\end{eqnarray*}
i.i.d. tasks each with $t_{H} = \Omega\round{\Delta^{-2}\sqrt{k}\log\round{\frac{k}{p_{\rm min}\Delta\delta}}}$ number of examples.

Finally we invoke Lemma~\ref{lem:refined-est}. Notice that in the last step we have estimated each $\w_i$ with error $\enorm{\tilde{\w}_i-\w_i} \le \enorm{\U\U^\top\tilde{\w}_i-\U\U^\top\w_i}+\enorm{\U\U^\top\w_i-\w_i}\le \Delta/10$. Hence the input for Lemma~\ref{lem:refined-est} satisfies $\enorm{\tilde{\w}_i-\w_i}\le \Delta/10$. It is not hard to verify that 
\begin{align*}
\round{1+\frac{\Delta^2}{50\rho^2}}\tilde{r}_i^2\ge  \round{s_i^2+\esqnorm{\tilde{\w}_i-\w_i}}\ge \round{1-\frac{\Delta^2}{50\rho^2}}\tilde{r}_i^2
\end{align*}
Hence, given  
\begin{eqnarray*}
n_{L2} \;\; =\;\; \Omega\round{\frac{  \,d\log^2(k/\delta)}{t_{L2}p_{\rm min}\eps^2 }}\; 
\end{eqnarray*}
i.i.d. tasks each with $t_{L2} = \Omega\round{\log\round{\frac{kd}{p_{\rm min}\delta\eps}}/\Delta^4}$ examples. We have parameter estimation with accuracy 
\begin{align*}
\enorm{\hat{\w}_i-\w_i} &\le \eps s_i\; , \\ 
\abs{\hat{s}_i^2-s_i^2}&\le  \frac{\eps}{\sqrt{d}}s_i^2\; ,\quad\text{and}\\
\abs{\hat{p}_i-p_i} &\le  \eps \sqrt{t_{L2}/d}p_{\min}.
\end{align*}
This concludes the proof.
\end{proof}

\subsection{Proof of Lemma~\ref{lem:projection}}
\label{sec:projection_proof} 
\begin{proposition}[Several facts for sub-Gaussian random variables] \label{prop:random-good-gaussian-mean}
Under our data generation model, let $c_1 > 1$ denote a sufficiently large constant, let $\delta \in (0,1)$ denote the failure probability. We have,  with probability $1-\delta$, for all $ i \in [ n ], j \in [ t ] $,
\[
\enorm{\frac{1}{t}\sum_{j=1}^t y_{i,j}\x_{i,j} - \beta_{i}}\le c_1 \cdot \sqrt{d} \cdot \rho \cdot \log(nd /\delta) \cdot t^{-1/2} . 
\]
\end{proposition}
\begin{remark}\label{rem:sqrt}
The above about is not tight, and can be optimized to $\log ( \cdot ) /t + \log^{1/2} ( \cdot ) / t^{1/2}$. Since we don't care about log factors, we only write $\log ( \cdot ) / t^{1/2}$ instead (note that $t \geq 1$). 
\end{remark}

\begin{proof}
For each $i\in [n], j\in [t], k\in [d]$, $y_{i,j}x_{i,j,k}$ is a sub-exponential random variable with sub-exponential norm $\|y_{i,j}x_{i,j,k}\|_{\psi_1} \le \sqrt{s_i^2+\esqnorm{\beta_i}} = \rho_i$. %\Zhao{What is the definition of $\psi_1$?}\Raghav{Defining it above. Looks good?} 

By Bernstein's inequality, 
\[
\Prob{\abs{\frac{1}{t}\sum_{j=1}^t y_{i,j}x_{i,j,k} - \beta_{i,k}}\ge z}\le 2\exp \left(-c \min\curly{\frac{z^2t}{\rho_i^2}, \frac{zt}{\rho_i}} \right)
\]
for some $c>0$. Hence we have that with probability $1-2\delta$, $\ \forall\ i\in \squarebrack{n}, k\in \squarebrack{d}$,
\[
\abs{\frac{1}{t}\sum_{j=1}^t y_{i,j}x_{i,j,k} - \beta_{i,k}}\le \rho_i\max\curly{\frac{\log \round{nd/\delta}}{ct}, \sqrt{\frac{\log \round{nd/\delta}}{ct}}},
\]
which implies 
\[
\enorm{\frac{1}{t}\sum_{j=1}^t y_{i,j}\x_{i,j} - \beta_{i}}\le \sqrt{d}\rho_i\max\curly{\frac{\log \round{nd/\delta}}{ct}, \sqrt{\frac{\log \round{nd/\delta}}{ct}}}.\qedhere
\]
\end{proof}
\begin{proposition}\label{prop:marginal-variance} For any $\vvec\in \mathbb{S}^{d-1}$
\[
\E{ \Big\langle\vvec,  \frac{1}{t}\sum_{j=1}^t y_{i,j}\x_{i,j} - \beta_{i} \Big\rangle^2 }\leq \BigO{\rho_i^2/t}.
\]
\end{proposition}
\begin{proof}
\begin{align*}
\E{ \Big\langle\vvec,  \frac{1}{t}\sum_{j=1}^t y_{i,j}\x_{i,j} - \beta_{i} \Big\rangle^2 } &= \frac{1}{t^2}\sum_{j=1}^t\sum_{j'=1}^t\E{\vvec^\top\round{y_{i,j}\x_{i,j} - \beta_{i}}\vvec^\top\round{y_{i,j'}\x_{i,j'} - \beta_{i}}}\\
&= \frac{1}{t^2}\sum_{j=1}^t\sum_{j'=1}^t\vvec^\top\E{\round{y_{i,j}\x_{i,j}-\beta_i}\round{y_{i,j'}\x_{i,j'}-\beta_i}^\top}\vvec
\end{align*}
where
\begin{align*}
    & ~ \E{\round{y_{i,j}\x_{i,j}-\beta_i}\round{y_{i,j'}\x_{i,j'}-\beta_i}^\top}\\
    = & ~ \E{\x_{i,j}\round{\x_{i,j}^\top\beta_i+\eps_{i,j}}\round{\beta_i^\top\x_{i,j'}+\eps_{i,j'}}\x_{i,j'}^\top - \round{\x_{i,j}^\top\beta_i+\eps_{i,j}}\x_{i,j}\beta_i^\top - \round{\x_{i,j'}^\top\beta_i+\eps_{i,j'}}\x_{i,j'}\beta_i^\top + \beta_i\beta_i^\top}\\
    = & ~ \E{\x_{i,j}\x_{i,j}^\top\beta_i\beta_i^\top\x_{i,j'}\x_{i,j'}^\top + \eps_{i,j}\eps_{i,j'}\x_{i,j}\x_{i,j'}^\top - \round{\x_{i,j}^\top\beta_i}^2 - \round{\x_{i,j'}^\top\beta_i}^2 + \beta_i\beta_i^\top}\\
    = & ~ \E{\x_{i,j}\x_{i,j}^\top\beta_i\beta_i^\top\x_{i,j'}\x_{i,j'}^\top - \beta_i\beta_i^\top} + \E{\eps_{i,j}\eps_{i,j'}\x_{i,j}\x_{i,j'}^\top}.
\end{align*}
Therefore, when $j\neq j'$,
\begin{align*}
    \E{\round{y_{i,j}\x_{i,j}-\beta_i}\round{y_{i,j'}\x_{i,j'}-\beta_i}^\top} &= 0.
\end{align*}
Plugging back we have
\begin{align*}
    \E{ \Big\langle\vvec,  \frac{1}{t}\sum_{j=1}^t y_{i,j}\x_{i,j} - \beta_{i} \Big\rangle^2 } &= \inv{t^2}\sumj{1}{t}\E{\round{\vvec^\top\x_{i,j}}^2\round{\beta_i^\top\x_{i,j}}^2 - \round{\vvec^\top\beta_i}}^2 + \vvec^\top\E{\eps_{i,j}^2\x_{i,j}\x_{i,j}^\top}\vvec\\
    &\leq \inv{t^2}\sumj{1}{t}\BigO{\esqnorm{\vvec}\esqnorm{\beta_i}} + \BigO{\vvec^\top\beta_i}^2 + s_i^2\esqnorm{\vvec}\\
    &\leq \BigO{ \rho_i^2 / t }.\qedhere
\end{align*}
\end{proof}
\begin{proposition}\label{pro:bound_norm_of_sum_over_t}
\[
\E{\esqnorm{\frac{1}{t}\sum_{j=1}^t y_{i,j}\x_{i,j} - \beta_{i}}}\leq \BigO{  \rho_i^2 d / t }
\]
\end{proposition}
\begin{proof}
\begin{align*}
    & ~ \E{ \Big\langle \inv{t}\sumj{1}{t}\round{y_{i,j}\x_{i,j}-\beta_i},  \inv{t}\sum_{j'=1}^{t}\round{y_{i,j'}\x_{i,j'}-\beta_i} \Big\rangle }\\
    = & ~ \inv{t^2}\sumj{1}{t}\sum_{j'=1}^{t}\E{y_{i,j}y_{i,j'}\x_{i,j}^\top\x_{i,j'} - \beta_i^\top y_{i,j'}\x_{i,j'} - \beta_i^\top y_{i,j}\x_{i,j} + \beta_i^\top\beta_i}\\
    = & ~ \inv{t^2}\sumj{1}{t}\sum_{j'=1}^{t}\E{y_{i,j}y_{i,j'}\x_{i,j}^\top\x_{i,j'} - \beta_i^\top\beta_i}\\
    = & ~ \inv{t^2}\sumj{1}{t}\sum_{j'=1}^{t} \E{\round{\beta_i^\top\x_{i,j}+\eps_{i,j}}\round{\beta_i^\top\x_{i,j'}+\eps_{i,j'}}\x_{i,j}^\top\x_{i,j'}-\esqnorm{\beta_i}}\\
    = & ~ \inv{t^2}\sumj{1}{t}\sum_{j'=1}^{t} \E{\beta_i^\top\x_{i,j}\x_{i,j}^\top\x_{i,j'}\x_{i,j'}^\top\beta_i + \eps_{i,j}\eps_{i,j'}\x_{i,j}^\top\x_{i,j'}-\esqnorm{\beta_i}}.
\end{align*}
The above quantity can be split into two terms, one is diagonal term, and the other is off-diagonal term.

If $j\neq j'$, then
\[
\E{\beta_i^\top\x_{i,j}\x_{i,j}^\top\x_{i,j'}\x_{i,j'}^\top\beta_i + \eps_{i,j}\eps_{i,j'}\x_{i,j}^\top\x_{i,j'}} -\esqnorm{\beta_i} = 0,
\]
and if $j=j'$, then
\[
\E{\beta_i^\top\x_{i,j}\x_{i,j}^\top\x_{i,j'}\x_{i,j'}^\top\beta_i + \eps_{i,j}\eps_{i,j'}\x_{i,j}^\top\x_{i,j'}-\esqnorm{\beta_i}} = \BigO{d\esqnorm{\beta_i}} + \sigma_i^2d = \BigO{\rho_i^2d}.
\]
Plugging back we get
\begin{align*}
    \E{\esqnorm{\frac{1}{t}\sum_{j=1}^t y_{i,j}\x_{i,j} - \beta_{i}}} 
    \leq & ~ \inv{t^2}\cdot t\cdot\BigO{\rho_i^2d}\\
    \leq & ~ \BigO{ \rho_i^2d / t }.
    \qedhere
\end{align*}
\end{proof}

\begin{definition}\label{def:Z_i}
For each $i\in [n]$, define matrix $\Z_i \in \R^{d \times d}$ as 
\[
\Z_i \coloneqq \round{\frac{1}{t} \sum_{j=1}^{t}y_{i,j}\x_{i,j}} \round{ \frac{1}{t} \sum_{j=t+1}^{2t} y_{i,j}\x_{i,j}^\top} - \beta_i\beta_i^\top.
\]
\end{definition}
We can upper bound the spectral norm of matrix $\Z_i$,

\begin{lemma}
Let $\Z_i$ be defined as Definition~\ref{def:Z_i}, let $c_2 > 1$ denote some sufficiently large constant, let $\delta \in (0,1)$ denote the failure probability. Then we have : with probability $1-\delta$,
\begin{align*}
    \forall\ i \in [n], ~~~\enorm{\Z_i} \leq c_2 \cdot d \cdot \rho_i^2 \cdot  \log^2 (nd/\delta) / {t}  
\end{align*}
\end{lemma}

\begin{proof}
The norm of $\enorm{\Z_i}$ satisfies
\begin{align*}
    \enorm{\Z_i} 
    \le & ~ \enorm{\round{\frac{1}{t} \sum_{j=1}^{t}y_{i,j}\x_{i,j}-\beta_i} \round{ \frac{1}{t} \sum_{j=t+1}^{2t} y_{i,j}\x_{i,j}^\top}} + \enorm{\beta_i\round{ \frac{1}{t} \sum_{j=t+1}^{2t} y_{i,j}\x_{i,j}^\top - \beta_i^\top}}\\
    \leq & ~ c_1 \sqrt{d}\rho_i \log(nd/\delta) t^{-1/2} \cdot \enorm{\frac{1}{t} \sum_{j=t+1}^{2t} y_{i,j}\x_{i,j}} +  c_1 \sqrt{d}\rho_i \log(n d/\delta)t^{-1/2} \cdot \enorm{\beta_i} \\
    = & ~ c_1 \sqrt{d}\rho_i \log(nd/\delta) t^{-1/2} \cdot \round{\enorm{\frac{1}{t} \sum_{j=t+1}^{2t} y_{i,j}\x_{i,j}} + \enorm{\beta_i}}\\
    \leq & ~ c_1 \sqrt{d}\rho_i \log(nd/\delta) t^{-1/2} \cdot \round{\enorm{\frac{1}{t} \sum_{j=t+1}^{2t} y_{i,j}\x_{i,j} - \beta_i} + 2 \enorm{\beta_i}}\\
    \leq & ~ c_1 \sqrt{d}\rho_i \log(nd/\delta) t^{-1/2}  \cdot \round{ \BigO{1} \cdot \sqrt{d} \rho_i \log(n d / \delta) t^{-1/2} + 2 \| \beta_i \|_2 }\\
    \leq & ~ \BigO{1} \cdot d \rho_i^2 \log^2 (nd/\delta) / t
\end{align*}
where the second step follows from Proposition~\ref{prop:random-good-gaussian-mean}, the fourth step follows from triangle inequality, the fifth step follows from Proposition~\ref{prop:random-good-gaussian-mean}, and the last step follows $\| \beta_i \|_2 \leq \rho_i$.

Rescaling the $\delta$ completes the proof.\qedhere
  
\end{proof}

\begin{definition}\label{def:event_truncation}
Let $c_2 > 1$ denote a sufficiently large constant. We define event $\Ecal$ to be the event that 
\begin{align*}
\forall\ i \in [n], ~~~ \enorm{\Z_i}\le c_2 \cdot d \cdot \rho^2 \cdot \log^2 (nd/\delta) / t . 
\end{align*}
\end{definition}

We can apply matrix Bernstein inequality under a conditional distribution.
\begin{proposition}
Let $\Z_i$ be defined as Definition~\ref{def:Z_i}. Let $\Ecal$ be defined as Definition~\ref{def:event_truncation}. Then we have
\begin{align*}
\left\| \E{\sum_{i=1}^n \Z_i \Z_i^\top \Big| \Ecal} \right\|_2 =  \BigO{ n \rho^4 d / t }.
\end{align*}
\end{proposition}
\begin{proof}
\begin{align*}
 & ~\enorm{\E{\Z_i \Z_i^\top}}  \\
= & ~ \max_{\vvec\in \mathbb{S}^{d-1}}\squarebrack{\E{\round{\vvec^\top \round{\frac{1}{t} \sum_{j=1}^{t}y_{i,j}\x_{i,j}}}^2\esqnorm{\frac{1}{t} \sum_{j=t+1}^{2t}y_{i,j}\x_{i,j}} - \round{\vvec^\top\beta_i}^2\esqnorm{\beta_i}}}\\
= & ~ \max_{\vvec\in \mathbb{S}^{d-1}}\squarebrack{\E{\round{\vvec^\top \round{\frac{1}{t} \sum_{j=1}^{t}y_{i,j}\x_{i,j}-\beta_i}}^2\esqnorm{\frac{1}{t} \sum_{j=t+1}^{2t}y_{i,j}\x_{i,j}}} + \E{\round{\vvec^\top\beta_i}^2\esqnorm{\round{\frac{1}{t} \sum_{j=t+1}^{2t}y_{i,j}\x_{i,j}}-\beta_i} }}\\
\lesssim & ~ (\rho_i^2 / t ) \cdot (\esqnorm{\beta_i}+ \rho_i^2d/t )+\esqnorm{\beta_i} ( \rho_i^2d / t )\\
\leq & ~ ( \rho_i^2 / t ) \cdot ( \rho_i^2 + \rho_i^2 d / t ) + \rho_i^2 \cdot ( \rho_i^2 d / t ) \\
\leq & ~ 2 \rho_i^4 d / t^2 + \rho_i^4 d / t \\
\leq & ~ 3\rho_i^4d / t.
\end{align*}
where the forth step follows from $\| \beta_i \|_2 \leq \rho_i$, the fifth step follows $d/t \geq 1$, and the last step follows from $t \geq 1$.

Thus, 
\begin{align*}
\left\| \E{\sum_{i=1}^n \Z_i \Z_i^\top|\Ecal} \right\|_2 \le \frac{1}{\Prob{\Ecal}}\left\| \E{\sum_{i=1}^n \Z_i \Z_i^\top} \right\|_2 = \BigO{ n \rho^4 d / t } .
\end{align*}
where $n$ comes from repeatedly applying triangle inequality.
\end{proof}

Applying matrix Bernstein inequality, we get 
\begin{lemma}\label{lem:bound-Z}
Let $\Z_i$ be defined as Definition~\ref{def:Z_i}.  
For any $\tilde{\epsilon} \in (0,1)$ and $\delta \in (0,1)$, if $$n =  \Omega\round{ \frac{d}{t}\log^2\round{{nd}/{\delta}}\max\curly{\inv{\tilde\eps^2},\inv{\tilde\eps}\log\frac{nd}{\delta} } },$$ then with probability at least $1 - \delta$,
\begin{align*}
\enorm{\frac{1}{n}\sum_{i=1}^n \Z_i} \le \tilde{\eps} \cdot \rho^2 .
\end{align*}
\end{lemma}
\begin{proof}
Recall that $\Ecal$ is defined as Definition~\ref{def:event_truncation}.

Using matrix Bernstein inequality (Proposition~\ref{prop:matrix-bernstein}), we get
for any $z > 0$,
\begin{align*}
\Prob{\enorm{\frac{1}{n}\sum_{i=1}^n \mathbf{Z}_{i}}\ge z ~\Big|~ \Ecal}\le d \cdot \exp  \round{ -\frac{ z^2n/2 }{\rho^4d/t + zcd \rho^2 \log^2(nd/\delta) /t }}.
\end{align*}
For $z=\tilde\eps\rho^2$, we get
\begin{align}
    \Prob{\enorm{\frac{1}{n}\sum_{i=1}^n \mathbf{Z}_{i}}\ge \tilde\eps\rho^2 ~\Big|~ \Ecal}\le d \cdot \exp  \round{ -\frac{ \tilde\eps^2n/2 }{d/t + \tilde\eps cd \log^2(nd/\delta) /t }}\label{eq:bernstein_on_Zi}
\end{align}
for some $c>0$. If we want to bound the right hand side of Equation~\eqref{eq:bernstein_on_Zi} by $\delta$, it is sufficient to have
\begin{align}
    \frac{ \tilde\eps^2n/2 }{d/t + \tilde\eps cd \log^2(nd/\delta) /t } &\geq \log\frac{nd}{\delta}\nonumber\\
    \text{or, }n &\gtrsim \frac{d}{t}\log^2\round{{nd}/{\delta}}\max\curly{\inv{\tilde\eps^2},\inv{\tilde\eps}\log\frac{nd}{\delta} }
\end{align}
Therefore, if $\tilde\eps\log(nd/\delta)\gtrsim 1$, we just need $n \gtrsim \frac{d}{\tilde\eps t}\log^3\round{{nd}/{\delta}}$, else we need $n\gtrsim \frac{d}{t\tilde\eps^2}\log^2(nd/\delta)$ thus completing the proof.
\end{proof}

\begin{lemma}\label{lem:bound_X_Xbar}
    If $\X = \inv{n}\sumi{1}{n}\beta_i\beta_i^\top$ where $\beta_i = \w_i$ with probability $p_i$, and $\M = \sumj{1}{k}p_i\w_i\w_i^\top$ as its expectation, then for any $\delta\in\round{0,1}$ we have
    \begin{equation}
        \Prob{\enorm{\X-\M}\leq \tilde\eps\rho^2 } \geq 1-\delta.
    \end{equation}
    if $n = \Omega\round{\frac{\log ^{3}(k/\delta)}{\tilde\eps^2}}$.
\end{lemma}
\begin{proof}
    Let $\tilde{p}_j = \inv{n}\sumi{1}{n}\indicator{\w_j=\beta_i}\ \forall\ j\in\squarebrack{k}$, then $\X = \sumj{1}{k}\tilde{p}_j\w_j\w_j^\top$. Let $\Smat_j = (\tilde{p}_j-p_j)\w_j\w_j^\top\ \forall j \in\squarebrack{k}$, then we have the following for all $j\in\squarebrack{k}$,
    \begin{align}
        \E{\Smat_j} &= \zero\nonumber\\
        \enorm{\Smat_j} &\leq \rho^2\sqrt{\frac{3\log(k/\delta)}{n}}\qquad(\text{from Proposition~\ref{prop:mult-con}})\\
        \enorm{\sumj{1}{k}\E{\Smat_j^\top\Smat_j}} &= \enorm{\sumj{1}{k} \E{\round{\tilde{p}_j-p_j}^2}\esqnorm{\w_j}\w_j\w_j^\top}\nonumber\\
        &\leq 3\rho^2\frac{\log(k/\delta)}{n}\enorm{\sumj{1}{k}p_j\w_j\w_j^\top}\nonumber\qquad(\text{from Proposition~\ref{prop:mult-con}})\\
        &\leq 3\rho^4\frac{\log(k/\delta)}{n}.
    \end{align}
    Conditioning on the event $\Ecal := \curly{\abs{\tilde{p}_j-p_j}\leq \sqrt{3\log(k/\delta)/n}}$, from matrix Bernstein~\ref{prop:matrix-bernstein} we have
    \begin{align}
        \Prob{\enorm{\sumj{1}{k}\Smat_j}\geq z\Bigm|\Ecal} &\leq 2k\exp\round{\frac{-z^2/2}{ 3\rho^4\frac{\log(k/\delta)}{n} + \frac{\rho^2z}{3}\sqrt{\frac{3\log(k/\delta)}{n}} }}\nonumber\\
        \implies \Prob{\enorm{\sumj{1}{k}\Smat_j}\leq 3\rho^2\frac{\log^{3/2}(k/\delta)}{\sqrt{n}}\Bigm|\Ecal} &\geq 1-\delta
    \end{align}
    Since $\Prob{\Ecal}\geq 1-\delta$, we have
    \begin{align}
        \Prob{\enorm{\sumj{1}{k}\Smat_j}\leq \tilde\eps\rho^2} &\geq 1-\delta
    \end{align}
    for $n = \Omega\round{\frac{\log^3(k/\delta)}{\tilde\eps^2}}$.
\end{proof}

\begin{lemma}\label{lem:lemma_5_in_ll18}
Given $k$ vectors $\x_1, \x_2, \cdots, \x_k \in \R^d$. For each $i \in [k]$, we define $\X_i = \x_i\x_i^\top$. For every $\gamma \ge 0$, and  every PSD matrix $\hat\M\in \Rd{d\times d}$ such that 
\begin{align}\label{eq:lem_assump}
\enorm{\hat\M-\sum_{i=1}^k \X_i}\;\; \le \;\;\gamma,
\end{align}
let $\U\in \Rd{d\times k}$ be the matrix consists of the top-$k$ singular vectors of $\hat\M$, then for all $i\in [k]$, 
\begin{align*}
    \enorm{\x_i^\top \round{\I - \U\U^\top}} &\;\;\le\;\; \min\curly{\, 
    {\gamma \|\x_i\|_2 }/{\sigma_{\rm min}}  \,,\, \sqrt{2} \round{\gamma \|\x_i\|_2 }^{1/3} %\,,\,  (2\eps \|\x_i\|^2)^{1/4} 
    } \;,
\end{align*}
where $\sigma_{\rm min}$ is the smallest non-zero singular value of $\sum_{i\in[k]} \X_i$.
\end{lemma} 

\begin{proof} 

From the gap-free Wedin's theorem in~\citep[Lemma~B.3]{AL16}, it follows that 
\begin{align}
    \enorm{ (\I-\U\U^\top) \V_j }  &\leq \gamma / \sigma_j \;,
    \label{eq:wedin}
\end{align}
where $\V_j=[\vvec_1\,\ldots\,\vvec_j]$ is the matrix consisting of the $j$ singular vectors of $\sum_{i'\in[k]}\X_{i'}$ corresponding to the top $j$  singular values, and $\sigma_j$ is the $j$-th singular value. 
To get the first term on the upper bound, notice that as $\x_i$ lie on the subspace spanned by $\V_j$ where $j$ is the rank of $\sum_{i'\in[k]}\X_{i'}$. It follows that
\begin{align*}
\enorm{\round{\I-\U\U^\top} \V_j \V_j^T \x_i}\leq \enorm{\x_i} \gamma/\sigma_j \leq \enorm{\x_i} \gamma/\sigma_{\rm min}.
\end{align*}

Next, we optimize over this choice of $j$ to get the tightest bound that does not depend on the singular values. 
\begin{align*}
    \esqnorm{\round{\I-\U\U^\top} \x_i} 
    &= \esqnorm{\round{\I-\U\U^\top} \V_j\V_j^\top \x_i} + \esqnorm{\round{\I-\U\U^\top} \round{\I-\V_j\V_j^\top} \x_i}\\
    &\leq ( \gamma^2 / \sigma_j^2 ) \esqnorm{\x_i} + \sigma_{j+1} \;,
\end{align*}
for any $j\in[k]$ where we used 
$\esqnorm{\round{\I-\V_j\V_j^\top}\x_i} \leq \sigma_{j+1}$. This follows from 
\begin{align*}
\sigma_{j+1} = \enorm{\round{\I-\V_j\V_j^\top}\sum_{i'\in[k]}\X_{i'}\round{\I-\V_j\V_j^\top}} \geq \enorm{ \round{\I-\V_j\V_j^\top}\x_i\x_i^\top \round{\I-\V_j\V_j^\top}}=\esqnorm{\round{\I-\V_j\V_j^\top}\x_i}. 
\end{align*}
 Optimal choice of $j$ minimizes the upper bound, which happens when the two terms are of similar orders. Precisely, we choose $j$ to be the largest index such that $\sigma_j\geq \gamma^{2/3}\enorm{\x_i}^{2/3}$ (we take  $j=0$ if $\sigma_1\leq \gamma^{2/3}\enorm{\x_i}^{2/3}$). This gives an upper bound of $ 2\gamma^{2/3}\enorm{\x_i}^{2/3}$. This bound is tighter by a factor of $k^{2/3}$ compared to a similar result from~\citep[Lemma~5]{li2018learning}, where this analysis  is based on.  \qedhere

%it follows that for all $i\in [k]$
%\begin{align*}
%    \|\x_i^\top (\I - \U\U^\top)\| &\;\;\le\;\; (\|\x_i\|\eps k)^{1/3}\;,
%\end{align*}
%where we redefine $\eps$ as the $\eps\|\X_i\|$ in \cite{li2018learning}.

%From the assumption \eqref{eq:lem_assump}, the following holds for any $j\in[k]$: 
%    \begin{eqnarray*}
%\frac{1}{\|\x_j\|^2}
%\Big| \, \x_j^\top (\I - \U\U^\top) \big({\cal P}_k(\M) -\sum_{i=1}^k \X_i \big) (\I - \U\U^\top) \x_j  \,\Big| 
%&\leq& 2 \epsilon\;, 
%    \end{eqnarray*}
%where ${\cal P}_k(\cdot)$ is the rank-$k$ projection of a matrix, 
%and we used the fact that 
%$\|{\cal P}_k(\M) -\sum_{i=1}^k \X_i \|_2 \leq 2 \|\M-\sum_{i=1}^k \X_i\|_2$.
%As $\sum_i \x_i\x_i^\top$ is PSD, we get that the errors can only accumulate and cannot cancel each other: 
%    \begin{eqnarray*}
%\frac{1}{\|\x_j\|^2}
%\sum_{i=1}^k \, \big( \x_j^\top (\I - \U\U^\top)  \x_i \,\big)^2 
%&\leq& 2 \epsilon\;, 
%    \end{eqnarray*}
%    which gives 
%    $(1/\|\x_j\|^2)\|\x_j^\top(\I-\U\U^\top)\|^4 \leq 2\epsilon $, and the desired bound.  
\end{proof}
\begin{proof}[Proof of Lemma~\ref{lem:projection}]
We combine Lemma~\ref{lem:lemma_5_in_ll18} and Lemma~\ref{lem:bound-Z} to compute the proof. 
Let $\eps>0$ be the minimum positive real such that for $\x_i = \sqrt{p_i}\w_i$, $\gamma = \tilde\eps\rho^2$, $\sigma_{\rm min} = \lambda_{\rm min}$, we have
\begin{align*}
    \sqrt{p_i}\enorm{\round{\I-\U\U^\top}\w_i} &\leq \min\curly{ \tilde\eps\rho^3\sqrt{p_i} / {\lambda_{\rm min}}, \sqrt{2} \cdot \tilde\eps^{1/3}\rho p_i^{1/6}} \leq \eps\rho\sqrt{p_i}
\end{align*}
The above equation implies that
\begin{align*}
     \tilde\eps &= \max\curly{\frac{\lambda_{\rm min} \eps}{\rho^2},\frac{p_{\rm min} \eps^3 }{2\sqrt{2}}}.
\end{align*}
Since $\enorm{ \sum_{i=1}^{k} \tilde{p}_i\w_i\w_i^\top-\sum_{i=1}^{k} p_i\w_i\w_i^\top } + \enorm{ \hat\M-\sum_{i=1}^{k} p_i\w_i\w_i^\top } \leq \BigO{\tilde{\eps}\rho^2}$ for $$n = \Omega\round{ \max\curly{\inv{\tilde\eps^2}\log^3(k/\delta),\frac{d}{t\tilde\eps^2}\log^2\round{{nd}/{\delta}},\frac{d}{t\tilde\eps}\log^3\round{{nd}/{\delta}}}}$$ from Lemma~\ref{lem:bound-Z} and Proposition~\ref{lem:bound_X_Xbar}, we get 
\begin{align*}
    \enorm{\round{\I-\U\U^\top}\w_i} &\leq \eps\rho\qquad\forall\ i \in\squarebrack{k}
\end{align*}
% \begin{align*}
% n = \Omega \round{  t^{-1} d\log^3(nd/\delta) \cdot\min \curly{ \epsilon^{-2} \rho^4\lambda_{\min}^{-2}, \epsilon^{-6} p_{\min}^{-2}  } }
% \end{align*}
with probability at least $1- \delta$.
\end{proof}

\subsection{Proof of Lemma~\ref{lem:distance}}
We start with the following two proposition which shows that the mean of our distance estimator is well separated between the in-cluster tasks and the inter-cluster tasks.
\begin{proposition}
Recall that matrix $\U$ satisfies Equation~\eqref{eq:projection} with error $\eps$. If $\Delta \geq 4\rho\eps$, then $\ \forall\ i, j\in [n]$ such that $\beta_i\ne \beta_j$,
\[
\E{\round{\hat\beta_i^{(1)} - \hat\beta_j^{(1)}}^\top\U \U^\top\U \U^\top\round{\hat\beta_i^{(2)} - \hat\beta_j^{(2)}}}\geq \Delta^2/4,
\]
and $\ \forall\ i, j \in [n]$ such that $\beta_i = \beta_j$, 
\[
\E{\round{\hat\beta_i^{(1)} - \hat\beta_j^{(1)}}^\top\U \U^\top\U \U^\top\round{\hat\beta_i^{(2)} - \hat\beta_j^{(2)}}}=0.
\]
\end{proposition}
\begin{proof}

If $\beta_i\ne \beta_j$, 
\begin{align*}
& ~\E{\round{\hat\beta_i^{(1)} - \hat\beta_j^{(1)}}^\top\U \U^\top\U \U^\top\round{\hat\beta_i^{(2)} - \hat\beta_j^{(2)}}} \\
= & ~ \esqnorm{\U\U^\top\round{\beta_i - \beta_j}} \\
= & ~ \esqnorm{\U\U^\top\beta_i-\beta_i + \beta_i-\beta_j + \beta_j - \U\U^\top\beta_j} \\
\ge & ~ \round{\enorm{\beta_i-\beta_j}-2\eps\rho}^2\\
\geq & ~  \Delta^2 / 4. 
\end{align*}
The proof is trivial for $\beta_i=\beta_j$.
\end{proof}

\begin{proposition}
\[
\Var{\round{\hat\beta_i^{(1)} - \hat\beta_j^{(1)}}^\top\U \U^\top\U \U^\top\round{\hat\beta_i^{(2)} - \hat\beta_j^{(2)}}} \le \BigO{ \rho^4 \cdot ( t + k ) / t^2 }.
\]
\end{proposition}
\begin{proof}
If $\beta_i \ne \beta_j$, then
\begin{align*}
&\Var{\round{\hat\beta_i^{(1)} - \hat\beta_j^{(1)}}^\top\U \U^\top\U \U^\top\round{\hat\beta_i^{(2)} - \hat\beta_j^{(2)}}}\\
&=\E{\round{\round{\hat\beta_i^{(1)} - \hat\beta_j^{(1)}}^\top\U\U^\top\round{\hat\beta_i^{(2)} - \hat\beta_j^{(2)}}}^2} - \round{\round{\beta_i-\beta_j}^\top\U\U^\top\round{\beta_i-\beta_j}}^2\\
&= \frac{1}{t^4}\sum_{\substack{a,a'=1\\b,b'=t+1}}^{t,2t}\E{\round{(y_{i,a}\x_{i,a} - y_{j,a}\x_{j,a})^\top\U\U^\top(y_{i,b}\x_{i,b} - y_{j,b}\x_{j,b})}\round{(y_{i,a'}\x_{i,a'} - y_{j,a'}\x_{j,a'})^\top\U\U^\top(y_{i,b'}\x_{i,b'} - y_{j,b'}\x_{j,b'})}} \\
& \qquad\qquad\qquad-(\beta_i-\beta_j)^\top\U\U^\top(\beta_i-\beta_j)(\beta_i-\beta_j)^\top\U\U^\top(\beta_i-\beta_j).
\end{align*}
For each term in the summation, we classify it into one of the $3$ different cases according to $a,b,a',b'$:
\begin{enumerate}
\item If $a\ne a'$ and $b\ne b'$, the term is $0$.
\item If $a=a'$ and $b\ne b'$, the term can then be expressed as:
\begin{align*}
& ~ \E{\round{(y_{i,a}\x_{i,a} - y_{j,a}\x_{j,a})^\top\U\U^\top(y_{i,b}\x_{i,b} - y_{j,b}\x_{j,b})}\round{(y_{i,a'}\x_{i,a'} - y_{j,a'}\x_{j,a'})^\top\U\U^\top(y_{i,b'}\x_{i,b'} - y_{j,b'}\x_{j,b'})}}\\
- & ~ (\beta_i-\beta_j)^\top\U\U^\top(\beta_i-\beta_j)(\beta_i-\beta_j)^\top\U\U^\top(\beta_i-\beta_j)\\
= & ~ \E{\round{(y_{i,a}\x_{i,a} - y_{j,a}\x_{j,a})^\top\U\U^\top(\beta_i-\beta_j)}^2} - \round{(\beta_i-\beta_j)^\top\U\U^\top(\beta_i-\beta_j)}^2\\
= & ~ \E{\round{y_{i,a}\x_{i,a}^\top\U\U^\top(\beta_i-\beta_j)}^2} - \round{\beta_i^\top\U\U^\top(\beta_i-\beta_j)}^2\\
& \qquad + \E{\round{y_{j,a}\x_{j,a}^\top\U\U^\top(\beta_i-\beta_j)}^2} - \round{\beta_j^\top\U\U^\top(\beta_i-\beta_j)}^2\\
= & ~ \BigO{\rho^4}.
\end{align*}
The last equality follows from the sub-Gaussian assumption of $\x$.
\item If $a\ne a'$ and $b=b'$, this case is symmetric to the last case and $3\sigma_{a}^2\sigma_{a'}^2$ is an upper bound.
\item If $a=a'$ and $b=b'$, the term can then be expressed as:
\begin{align*}
& ~\E{\round{(y_{i,a}\x_{i,a} - y_{j,a}\x_{j,a})^\top\U\U^\top(y_{i,b}\x_{i,b} - y_{j,b}\x_{j,b})}^2} - \round{(\beta_i-\beta_j)^\top\U\U^\top(\beta_i-\beta_j)}^2\\
= & ~ \E{y_{i,b}^2((y_{i,a}\x_{i,a} - y_{j,a}\x_{j,a})^\top\U\U^\top \x_{i,b})^2} + \E{y_{j,b}^2((y_{i,a}\x_{i,a} - y_{j,a}\x_{j,a})^\top\U\U^\top \x_{j,b})^2}\\
& \qquad\qquad - 2\E{(y_{i,a}\x_{i,a} - y_{j,a}\x_{j,a})^\top\U\U^\top (y_{i,b}\x_{i,b})(y_{i,a}\x_{i,a} - y_{j,a}\x_{j,a})^\top\U\U^\top (y_{j,b}\x_{j,b})}\\
& \qquad\qquad\qquad - \round{(\beta_i-\beta_j)^\top\U\U^\top(\beta_i-\beta_j)}^2.
\end{align*}
First taking the expectation over $\x_{i,b},y_{i,b}, \x_{j,b},y_{j,b}$, we get the following upper bound 
\begin{align*}
&c_3\rho^2\E{\esqnorm{(y_{i,a}\x_{i,a} - y_{j,a}\x_{j,a})^\top\U\U^\top}} - 2\E{(y_{i,a}\x_{i,a} - y_{j,a}\x_{j,a})^\top\U\U^\top \beta_i(y_{i,a}\x_{i,a} - y_{j,a}\x_{j,a})^\top\U\U^\top \beta_j}
\end{align*}
for some $c_3>0$. Since 
\[
\E{(y_{i,a}\x_{i,a} - y_{j,a}\x_{j,a})^\top\U\U^\top \beta_i(y_{i,a}\x_{i,a} - y_{j,a}\x_{j,a})^\top\U\U^\top \beta_j} \lesssim \rho^2\E{\esqnorm{(y_{i,a}\x_{i,a} - y_{j,a}\x_{j,a})^\top\U\U^\top}},
\]
we have the following upper bound:
\begin{align*}
&\lesssim \E{\esqnorm{(y_{i,a}\x_{i,a} - y_{j,a}\x_{j,a})^\top\U\U^\top}}\\
&\lesssim \E{\esqnorm{(y_{i,a}\x_{i,a})^\top\U}} + \E{\esqnorm{(y_{j,a}\x_{j,a})^\top\U}}.
\end{align*}
Since $\E{\round{(y_{i,a}\x_{i,a})^\top\uvec_l}^2}\leq \BigO{\rho^2}\ \forall\ l\in [k]$, we finally have a $\BigO{k}$ upper bound for this case.
\end{enumerate}
The final step is to sum the contributions of these $4$ cases. Case $2$ and $3$ have $\BigO{t^3}$ different quadruples $(a,b,a',b')$. Case $4$ has $\BigO{t^2}$ different quadruples $(a,b,a',b')$. Combining the resulting bounds yields an upper bound of:
\[
\BigO{\rho^4\cdot( t+k)/ t^2 }.\qedhere
\]
\end{proof}
We now have all the required ingredients for the proof of Lemma~\ref{lem:distance}
\begin{proof}[Proof of Lemma~\ref{lem:distance}]
For each pair $i,j$, we repeatedly compute 
\[
\round{\hat\beta_i^{(1)} - \hat\beta_j^{(1)}}^\top\U \U^\top\U \U^\top\round{\hat\beta_i^{(2)} - \hat\beta_j^{(2)}}
\]
$\log(n/\delta)$ times, each with a batch of new sample of size $\rho^2\sqrt{k}/\Delta^2$, and take the median of these estimates. With probability $1-\tilde\delta$, it holds that for all $\beta_i\ne\beta_j$, the median is greater than $c{\Delta^2}$, and for all $\beta_i = \beta_j$ the median is less than $c{\Delta^2}$ for some constant $c$. Hence the single-linkage algorithm can correctly identify the $k$ clusters.

%this procedure makes no error, and we classified all the examples correctly.
%$\beta_i\ne \beta_j$, and $\beta_i=\beta_j$ otherwise.

% Estimation of p

Conditioning on the event of perfect clustering, the cluster sizes are distributed according to a multinomial distribution, which from Proposition~\ref{prop:mult-con} can be shown to concentrate as
% \Weihao{Should be $1/n$ instead of $1/nt$ as follows?}
\begin{align*}
    \abs{p_i-\tilde{p}_i} \leq \sqrt{ \frac{3\log (k / \delta )}{n}p_i }
    \leq p_i/2
\end{align*}
with probability at least $1-\delta$ by our assumption that $n=\Omega\round{\frac{\log(k/\delta)}{p_{\rm min}}}$, which implies that $\hat{p}_i\ge p_i/2.$
% Estimation of \W

For each group, we compute the corresponding average of ${\U^\top\hat\beta_i}$ as
\begin{align*}
    \U^\top\tilde\w_l \coloneqq \inv{n\tilde{p}_lt}\sum\limits_{i \ni \beta_i = \w_l}\sumj{1}{t}y_{i,j}\U^\top\x_{i,j},
\end{align*}
which from Proposition~\ref{prop:random-good-gaussian-mean} would satisfy
\begin{align*}
    \enorm{\U^\top\round{\tilde\w_l-\w_l}} &\lesssim \sqrt{k}\rho_i\max\curly{\frac{\log (k^2/\delta)}{n\tilde{p}_lt},\sqrt{\frac{\log (k^2/\delta)}{n\tilde{p}_lt}}}\\
    &\leq \tilde\eps \rho_i.
\end{align*}
The last inequality holds due to the condition on $n$.

% Estimation of r_l

The estimate for $r_l^2 \coloneqq s_l^2 + \esqnorm{\w_l-\tilde\w_l}\ \forall\ l\in\squarebrack{k}$ is
\begin{align*}
    \tilde{r}_l^2 = \inv{n\tilde p_l t}\sum\limits_{i\ni \beta_i=\w_l}\sumj{1}{t}\round{\x_{i,j}^\top\round{\w_l-\tilde\w_l} + \eps_{i,j}}^2
\end{align*}

where $\x_{i,j}$ and $y_{i,j}$ are fresh samples from the same tasks. The expectation of $\hat{r}_l^2$ can be computed as
\begin{align*}
    \E{\tilde{r}_l^2} &= \inv{n\tilde p_l t}\sum\limits_{i\ni \beta_i=\w_i}\sumj{1}{t}\E{\round{\x_{i,j}^\top\round{\w_l-\tilde\w_l} + \eps_{i,j}}^2}\\
    &= s_l^2 + \esqnorm{\w_l-\tilde\w_l} = r_l^2
\end{align*}
We can compute the variance of $\tilde r_l^2$ like
\begin{align*}
    \Var{\tilde{r}_l^2} &= \inv{n\tilde p_l t}\sum\limits_{i\ni \beta_i=\w_i}\sumj{1}{t}\Var{\round{\x_{i,j}^\top\round{\w_l-\tilde\w_l} + \eps_{i,j}}^2}\\
    &= \inv{n\tilde p_l t}\sum\limits_{i\ni \beta_i=\w_i}\sumj{1}{t}\squarebrack{\E{\round{\x_{i,j}^\top\round{\w_l-\tilde\w_l} + \eps_{i,j}}^4} - \round{s_l^2+\esqnorm{\w_l-\tilde\w_l}}^2}
\end{align*}
Since $\round{\x_{i,j}^\top\round{\w_l-\tilde\w_l} + \eps_{i,j}}^2$ is a sub-exponential random variable, we can use Bernstein's concentration inequality to get
\begin{align*}
    \Prob{\abs{\tilde r_l^2 - r_l^2}>z} &\leq 2\exp\curly{-\min\curly{\frac{z^2t}{r_l^4},\frac{zt}{r_l^2}}}\\
    \implies \abs{\tilde r_l^2 - r_l^2} &<r_l^2\max\curly{\sqrt{\frac{\log\inv{\delta}}{n\tilde p_l t}},\frac{\log\inv{\delta}}{n\tilde p_l t}}\qquad \text{with probability at least $1-\delta$,}\\
    &\leq r_l^2\frac{\tilde\eps}{\sqrt{k}}
\end{align*}
where the last inequality directly follows from the condition on $n$.
\end{proof}

\subsection{Proof of Lemma~\ref{lem:refined-est}}
Before proving Lemma~\ref{lem:refined-est}, we first show that with the parameters $\w_i, r_i^2$ estimated with accuracy stated, for all $i\in\squarebrack{k}$ in the condition of Lemma~\ref{lem:refined-est}, we can correctly classify a new task using only $\Omega\round{\log k}$ dependency of $k$ on the number of examples $t_{\out}$.

\begin{lemma}[Classification]\label{lem:growing-label}
Given estimated parameters satisfying $\enorm{\tilde{\w}_i-\w_i} \le \Delta/10$,  $(1-\Delta^2/50)\tilde{r}_i^2\le s_i^2+\esqnorm{\tilde{\w}_i-\w_i}\le (1+\Delta^2/50)\tilde{r}_i^2$ for all $i\in\squarebrack{k}$, and a new task with $t_{\out}\ge \Theta\round{\log (k/\delta)/\Delta^4}$ samples whose true regression vector is $\beta = \w_h$, our algorithm predicts $h$ correctly with probability $1-\delta$.
\end{lemma}
\begin{proof}
Given a new task with $t_{\out}$ training examples, $\x_i,\; y_i = \w^\top\x_i+\eps_i$ for $i\in [t_{\out}]$ where the true regression vector is $\beta = \w_{h}$ and the true variance of the noise is $\sigma^2 = s^2_{h}$. Our algorithm compute the the following ``log likelihood'' like quantity with the estimated parameters, which is defined to be
\begin{align}
\hat{l}_i \coloneqq &-\sum_{j=1}^{t_{\out}} \round{y_j-\x_j^\top\tilde\w_i}^2 / \round{ 2\tilde{r}_i^2 } + t_{\out} \cdot \log\round{1/\tilde{r}_i}\label{eqn:lem:growing-label-likelihood}\\
=& -\sum_{j=1}^{t_{\out}} \round{\eps_j+\x_j^\top\round{\w_h-\tilde{\w}_i}}^2 / \round{ 2\tilde{r}_i^2 } +t_{\out} \cdot \log(1/\tilde{r}_i)\nonumber,
\end{align}
and output the classification as $\argmax_{i\in\squarebrack{k}} \hat{l}_i$. 

Our proof proceeds by proving a lower bound on the likelihood quantity of the true index $\hat{l}_h$, and an upper bound on the likelihood quantity of the other indices $\hat{l}_i$ for $i\in [k]\backslash \{h\}$, and we then argue that the $\hat{l}_h$ is greater than the other $\hat{l}_i$'s for $i\in [k]\backslash \{h\}$ with high probability, which implies our algorithm output the correct classification with high probability.

The expectation of $\hat{l}_h$ is
\begin{align*}
\E{ \hat{l}_h } = - t_{\out}\cdot\round{s_h^2+\esqnorm{\w_h-\tilde{\w}_h}} /\round{2\tilde{r}_h^2} +t_{\out} \cdot \log(1/\tilde{r_h}).
\end{align*}
Since $\round{\eps_j+\x_j^\top\round{\w_h-\tilde{\w}_h}}^2 / \round{ 2\tilde{r}_h^2}$ is a sub-exponential random variable with sub-exponential norm at most $\BigO{\round{ s_h^2 + \esqnorm{\w_h-\tilde{\w}_h} } / \tilde{r}_h^2} = \BigO{r_h^2/\tilde{r}_h^2}$, we can apply Bernstein inequality~\citep[Theorem~2.8.1]{vershynin2018high} to $\hat{l}_h$ and get
\begin{align*}
\Prob{\abs{\hat{l}_h - \E{ \hat{l}_h }} > z }\le 2 \exp \curly{ -c\min \curly{ \frac{z^2}{t_{\out}r_h^4/\tilde{r}_h^4},\frac{z}{r_h^2/\tilde{r}_h^2}} },
\end{align*}
which implies that with probability $1-\delta/k$,
\begin{align*}
\abs{\hat{l}_h-\E{\hat{l}_h} } \lesssim  r_h^2  / \tilde{r}_h^2  \cdot \max \curly{ \sqrt{t_{\out}\log(k/\delta)} , \log(k/\delta) }.
\end{align*}

Using the fact that $t_{\out}\ge C\log(k/\delta)$ for some $C>1$, we have that with probability $1-\delta/k$,
\begin{align*}
\hat{l}_h \ge - \round{t_{\out}+c\sqrt{t_{\out}\log(k/\delta)} } \cdot  r_h^2  / \round{ 2\tilde{r}_h^2 } +t_{\out} \cdot \log(1/\tilde{r}_h)
\end{align*}
for some constant $c>0$.
%and the variance is at most
%$$
%\Var{\hat{L}_i} \lesssim t_{test}\frac{(\sigma^2+\eps^2)^2}{4\hat{\sigma_i}^4}
%$$

For $i \ne h$, the expectation of $\hat{l}_i$ is at most
\begin{align*}
\E{ \hat{l}_i } &\le - t_{\out} \cdot \round{ s_i^2+\round{\Delta-\enorm{\w_i-\tilde{\w}_i}}^2 } / \round{ 2\tilde{r}_i^2 } + t_{\out} \cdot \log\round{1/\tilde{r}_i}.
\end{align*}
Since $\round{\eps_i+\x_j^\top\round{\w_h-\tilde{\w}_i}}^2 / \round{ 2\tilde{r}_i^2}$ is a sub-exponential random variable with sub-exponential norm at most $\BigO{ \round{s_i^2 + \round{\Delta+\enorm{\w_i-\tilde{\w}_i}}^2} / \tilde{r}_i^2}$. Again we can apply Bernstein's inequality and get with probability $1-\delta$
\begin{align*}
\hat{l}_i 
&\le ~ -t_{\out} \cdot \round{ s_i^2+\round{\Delta-\enorm{\w_i-\tilde{\w}_i}}^2 } / \round{ 2\tilde{r}_i^2 } +t_{\out}\log\round{1/\tilde{r}_i}\\
&\qquad\qquad\qquad+ c\sqrt{t_{\out}\log(k/\delta)} \cdot \round{ s_i^2+\round{\Delta+\enorm{\w_i-\tilde{\w}_i}}^2 } / \round{2\tilde{r}_i^2}
\end{align*}
for a constant $c>0$.

Using our assumption that $\enorm{\w_i-\tilde{\w}_i}\le \Delta/10$ for all $i\in [k]$, we get
\begin{align*}
\hat{l_i}\le & ~ \round{-t_{\out}+c'\sqrt{t_{\out}\log(k/\delta)}} \cdot \round{s_i^2+ 0.5 \Delta^2} / \round{2\tilde{r}_i^2} + 0.5  t_{\out} \log\round{1/\tilde{r}_i^2}
\end{align*}
for some constant $c'>0$. We obtain a worst case bound by taking the maximum over all possible value of $\tilde{r}_i$ as
\begin{align*}
\hat{l}_i \le & ~ - 0.5 t_{\out} - 0.5  t_{\out} \log\round{\round{1-c'\sqrt{{\log(k/\delta)} / {t_{\out}}}}\round{s_i^2+ 0.5  \Delta^2 }},
\end{align*}
where we have taken the maximum over all possible values of $\hat{r}_i$.

Using the assumption that 
\begin{align*}
r_h^2/\tilde{r}_h^2\le 1+\Delta^2/50
\end{align*}
and $t_{\out}\ge C\log(k/\delta)$ for some constant $C > 1$, we obtain that 
\begin{align*}
&-t_{\out} \cdot  r_h^2  / ( 2\tilde{r}_h^2 ) + 0.5 t_{\out} \ge t_{\out}\Delta^2/100, \quad\text{and}\\
&-c\sqrt{t_{\out}\log(k/\delta)} \cdot  r_h^2  / \round{ 2\tilde{r}_h^2 } +  0.5  t_{\out} \log\round{1-c'\sqrt{{\log(k/\delta)} / {t_{\out}}}} = \BigO{\sqrt{t_{\out}\log(k/\delta)}}.
\end{align*}
Further notice that 
$$
\round{1+\Delta^2/5}\tilde{r}_h^2\le \frac{\round{1+\Delta^2/5}}{1-\Delta^2/50}\round{s_h^2+\Delta^2/100} \le s_h^2+\Delta^2/2.
$$
since $s_h^2\leq 1$, and $\Delta\leq 2$. Plugging in these facts into $\hat{l}_h-\hat{l}_i$ and applying the assumption that $\round{s_h^2+\Delta^2/2} / \tilde{r}_h^2 \ge \round{1+\Delta^2/5}$ we get 
\begin{align*}
\hat{l}_h - \hat{l}_i\ge 0.5  t_{\out} \log\round{1+\Delta^2/5} -   t_{\out}\Delta^2/100 - \BigO{\sqrt{t_{\out}\log(k/\delta)}}
\end{align*}
By the fact that $\log\round{1+\Delta^2/5} - \Delta^2/50 \ge \Delta^2/5000$ for all $\Delta\le 50$, the above quantity is at least
\begin{align}
\Theta\round{t_{\out}\Delta^2}-\Theta\round{\sqrt{t_{\out}\log(k/\delta)}}\label{eqn:lem:growing-label-lwerbound}.
\end{align}
Since $t_{\out}\ge\Theta\round{ \log(k/\delta)/\Delta^4}$, we have that with probability $\delta$, for all $i\in [k]\backslash \{h\}$, it holds that  $\hat{l}_h-\hat{l}_i>0$, which implies the correctness of the classification procedure.
\end{proof}

\begin{proof}[Proof of Lemma~\ref{lem:refined-est}]
Given $n$ i.i.d. samples from our data generation model, by the assumption that $n= \Omega\round{\frac{d\log^2(k/\delta)}{p_{\rm min}\eps^2t}} =\Omega\round{\frac{\log(k/\delta)}{p_{\min}}}$ and from Proposition~\ref{prop:mult-con}, it holds that the number of tasks such that $\beta = \w_i$ is $n\hat{p}_i\ge \frac{1}{2}np_i$ with probability at least $1-\delta$. Hence, with this probability, there exists at least $np_i/10$ i.i.d. examples for estimating $\w_i$ and $s_i^2$. By Proposition~\ref{prop:ols-error}, it holds that with probability $1-\delta$, for all $i\in [k]$, our estimation satisfies
\begin{align*}
\esqnorm{\hat{\w}_i-\w_i} &= \BigO{\frac{\sigma^2\round{d+\log (k/\delta)}}{np_it}},\quad \text{and}\\
\abs{\hat{s}_i^2-s_i^2} &= \BigO{\frac{\log(k/\delta)}{\sqrt{np_it-d}}s_i^2}.
\end{align*}
By Proposition~\ref{prop:mult-con}, it holds that $$
|\hat{p}_i-p_i| \le \sqrt{\frac{3\log(k/\delta)}{n}p_i}
$$
Since $n= \Omega\round{\frac{d\log^2(k/\delta)}{p_{\rm min}\eps^2t}}$, we finally get for all $i\in [k]$
\begin{align*}
\enorm{\hat{\w}_i-\w_i} &\le \eps s_i\; ,\\
\abs{\hat{s}_i^2-s_i^2}&\le \frac{\eps s_i^2}{\sqrt{d}}\;,\quad\text{and}\\
\abs{\hat{p}_i-p_i}&\le \min\curly{p_{\min}/10,\eps p_i\sqrt{t/d}}.\qedhere
\end{align*}
\end{proof}

\section{Proof Theorem~\ref{thm:main_reg_pred}}
\label{sec:proof_pred}

We first bound the expected error of the maximum a posterior (\textit{MAP}) estimator.
\begin{lemma}\label{lem:MAP-error}
Given estimated parameters satisfying $\enorm{\hat{\w}_i-\w_i} \le \Delta/10$,  $\round{1-\Delta^2/50}\hat{s}_i^2\le s_i^2+\esqnorm{\hat{\w}_i-\w_i}\le \round{1+\Delta^2/50}\hat{s}_i^2$ for all $i\in\squarebrack{k}$, and a new task with $\tau\ge \Theta\round{\log (k/\delta)/\Delta^4}$ samples $\mathcal{D} = \{ \x_i,y_i\}_{i=1}^\tau$. Define the maximum a posterior (\textit{MAP}) estimator as
$$
\hat{\beta}_{\rm MAP}({\cal D}) \coloneqq \hat{\w}_{\hat{i}}
$$
where 
$$
\hat{i} \coloneqq \argmax_{i\in\squarebrack{k}} \round{\sum_{j=1}^\tau\frac{-\round{y_j-\hat\w_i^\top \x_j}^2}{2\hat{\sigma}_i^2}+\tau\log\round{1/\hat{\sigma}_i}+\log\round{\hat{p}_i}}.
$$
Then, the expected error of the MAP estimator is bound as
\begin{align*}
&\myE_{\cal T^{\rm new} \sim \P{(\cal T)}}\myE_{\cal D \sim \cal T^{\rm new}}\myE_{\{\x,y\}\sim \cal T^{\rm new}}\squarebrack{\round{\x^\top \hat{\beta}_{\rm MAP}({\cal D})-y}^2}\\
\le& \delta+\sum_{i=1}^kp_i\esqnorm{\w_i-\hat{\w}_i}+\sum_{i=1}^kp_is_i^2
\end{align*}
\end{lemma}
\begin{proof}
The proof is very similar to the proof of Lemma~\ref{lem:growing-label}. The log of the posterior probability given the training data $\cal D$ under the estimated parameters is
\begin{align}
\hat{l}_i \coloneqq &-\sum_{j=1}^{\tau} \round{y_j-\x_j^\top\hat\w_i}^2 / \round{ 2\hat{s}_i^2 } + \tau \cdot \log\round{1/\hat{s}_i} + \log\round{\hat{p}_i}\label{eqn:lem:MAP-li},
\end{align}
which is different from Equation~\ref{eqn:lem:growing-label-likelihood} just by a $\log(1/\hat{p}_i)$ additive factor. Hence, given that the true regression vector of the new task ${\cal T}^{\rm new}$ is $\w_h$, it follows from Equation~\ref{eqn:lem:growing-label-lwerbound} that $\hat{l}_h - \hat{l}_i$ with probability at least $1-\delta$ is greater than  
\begin{align*}
\Theta(\tau\Delta^2)-\Theta\round{\sqrt{\tau\log(k/\delta)}} + \log\round{\hat{p}_h/\hat{p}_i},
\end{align*}
which under the assumption that $\abs{\hat{p}_i-p_i}\le p_i/10$ is greater than 
\begin{align}
\Theta(\tau\Delta^2)-\Theta\round{\sqrt{\tau\log(k/\delta)}} - \log(1/p_h) - \log(10/9).\label{eqn:lem:MAP-diff}
\end{align}
If $p_h\ge \delta/k$, by our assumption that $\tau \ge\Theta\round{ \log(k/\delta)/\Delta^4}$, it holds that $\hat{l}_h-\hat{l}_i>0$ for all $i\ne h$, and hence the MAP estimator output $\hat{\w}_h$ with probability at least $1-\delta$. With the remaining less than $\delta$ probability, the MAP estimator output $\hat{\beta}_{\rm MAP} = \hat{\w}_i$ for some other $i\ne h$ which incurs $\ell_2$ error $\|\hat{\beta}_{\rm MAP} - \w_h\|_2\le\|\hat{\beta}_{\rm MAP}\|_@ +\|\w_h\|_2\le 2$.

If $p_h\le \delta/k$, we pessimistically bound the error of $\hat{\beta}_{\rm MAP}$ by $\|\hat{\beta}_{\rm MAP} - \w_h\|\le 2$.

To summarize, notice that 
\begin{align*}
&\myE_{\cal T^{\rm new} \sim \P{(\cal T)}}\myE_{\cal D \sim \cal T^{\rm new}}\myE_{\{\x,y\}\sim \cal T^{\rm new}}\squarebrack{\round{\x^\top \hat{\beta}_{\rm MAP}({\cal D})-y}^2}\\
=& \myE_{\cal T^{\rm new} \sim \P{(\cal T)}}\myE_{\cal D \sim \cal T^{\rm new}}\squarebrack{\esqnorm{\hat{\beta}_{\rm MAP}({\cal D})-\w_h}+s_h^2}\\
\le& \sum_{i=1}^k p_i\round{\indicator{p_i\ge \delta/k}\round{4\delta+(1-\delta)\esqnorm{\w_i-\hat{\w}_i}}}+\sum_{i=1}^k 4p_i\indicator{p_i\le \delta/k}+\sum_{i=1}^kp_is_i^2\\
\le& 4\delta+\sum_{i=1}^kp_i\|\w_i-\hat{\w}_i\|^2+4\delta+\sum_{i=1}^kp_is_i^2\\
=& 8\delta+\sum_{i=1}^kp_i\|\w_i-\hat{\w}_i\|^2+\sum_{i=1}^kp_is_i^2.
\end{align*}
Replacing $8\delta$ by $\delta$ concludes the proof.
\end{proof}
Next, we bound the expected error of the posterior mean estimator.
\begin{lemma}\label{lem:Bayes-error}
Given estimated parameters satisfying $\enorm{\hat{\w}_i-\w_i} \le \Delta/10$,  $s_i^2+\esqnorm{\hat{\w}_i-\w_i}\le (1+\Delta^2/50)\hat{s}_i^2$, $s_i^2+\Delta^2/2\ge (1+\Delta^2/5)\hat{s}_i^2$ for all $i\in\squarebrack{k}$, and a new task with $\tau\ge \Theta\round{\log (k/\delta)/\Delta^4}$ samples $\mathcal{D} = \{ \x_i,y_i\}_{i=1}^\tau$. Define the posterior mean estimator as
$$
\hat{\beta}_{\rm Bayes}({\cal D}) \coloneqq \frac{\sum_{i=1}^k \hat{L}_i\hat{\w}_{i}}{\sum_{i=1}^k \hat{L}_i}
$$
where 
$$
\hat{L}_{i} \coloneqq \exp \round{-\sum_{i=1}^\tau\frac{\round{y_j-\w_i^\top \x_j}^2}{2\hat{\sigma}_i^2}+\tau\log(1/\hat{\sigma}_i)+\log(\hat{p}_i)}.
$$

Then, the expected error of the posterior mean estimator is bound as
\begin{align*}
&\mathbb{E}_{\cal T^{\rm new} \sim \P{(\cal T)}}\mathbb{E}_{\cal D \sim \cal T^{\rm new}}\mathbb{E}_{\{\x,y\}\sim \cal T^{\rm new}}\squarebrack{\round{\x^\top \hat{\beta}_{\rm Bayes}({\cal D})-y}^2}\\
\le& \delta+\sum_{i=1}^kp_i\esqnorm{\w_i-\hat{\w}_i}+\sum_{i=1}^kp_is_i^2
\end{align*}
\end{lemma}
\begin{proof}
This proof is very similar to the proof of Lemma~\ref{lem:MAP-error}. Notice that 
\begin{align*}
&\myE_{\cal T^{\rm new} \sim \P{(\cal T)}}\myE_{\cal D \sim \cal T^{\rm new}}\myE_{\{\x,y\}\sim \cal T^{\rm new}}\squarebrack{\round{\x^\top \hat{\beta}_{\rm Bayes}({\cal D})-y}^2}\\
=& \myE_{\cal T^{\rm new} \sim \P{(\cal T)}}\myE_{\cal D \sim \cal T^{\rm new}}\squarebrack{\esqnorm{\hat{\beta}_{\rm Bayes}({\cal D})-\w_h}+s_h^2}
\end{align*}
where $\w_h$ is defined to be the true regression vector of the task $\cal T^{\rm new}$. 
\begin{align}
&\esqnorm{\hat{\beta}_{\rm Bayes}({\cal D}) - \w_h}\nonumber\\
 \le& \round{\enorm{\hat{\w}_h-\w_h}+\round{1-\frac{\hat{L}_h}{\sum_{i=1}^k\hat{L}_i}}\enorm{\w_h}+\sum_{j\ne h} \frac{\hat{L}_j}{\sum_{i=1}^k\hat{L}_i}\enorm{\w_j}}^2\nonumber\\
 \le& \round{\enorm{\hat{\w}_h-\w_h}+2\round{1-\frac{\hat{L}_h}{\sum_{i=1}^k \hat{L}_i}}}^2\nonumber\\
 \le& \round{\enorm{\hat{\w}_h-\w_h}+2\sum_{i\ne h} \hat{L}_i/\hat{L}_h}^2\label{eqn:lem:Bayes-err}
\end{align}
Notice that 
\begin{align*}
\hat{L}_i/\hat{L}_h = \exp(\hat{l}_i-\hat{l}_h)
\end{align*}
where $l_i$ is the logarithm of the posterior distribution as defined in Equation~\ref{eqn:lem:MAP-li}. Therefore we can apply  Equation~\ref{eqn:lem:MAP-diff} and have that with probability $\delta$, 
$$
\hat{l}_i-\hat{l}_h \le -\log(k/\delta)/\Delta^2 \le -\log(k/\delta)
$$
for $\tau = \Omega(\log(k/\delta)/\Delta^4)$, which is equivalent to 
$$
\hat{L}_i/\hat{L}_h \le \delta/k.
$$
Plugging this into Equation~\ref{eqn:lem:Bayes-err} yields for a fixed $\cal T^{\rm new}$, with probability $1-\delta$, \begin{align*}
\esqnorm{\hat{\beta}_{\rm Bayes}({\cal D}) - \w_h} \le& \round{\enorm{\hat{\w}_h-\w_h}+2\sum_{i\ne h} \hat{L}_i/\hat{L}_h}^2\\
\le& \esqnorm{\hat{\w}_h-\w_h}+4\delta^2+4\delta \enorm{\hat{\w}_h-\w_h}\\
\le& \esqnorm{\hat{\w}_h-\w_h}+ 8\delta,
\end{align*}
and the error is at most $4$ for the remaining probability $\delta$. 
Hence we get for a fixed $\cal T^{\rm new}$
$$
\mathbb{E}_{{\cal D}\sim \cal T^{\rm new}}\squarebrack{\esqnorm{\hat{\beta}_{\rm Bayes}({\cal D})-\w_h}+s_h^2} \le \esqnorm{\hat{\w}_h-\w_h}+ s_h^2+12\delta.
$$
Finally taking the randomess of $\cal T^{\rm new}$ into account, we have 
\begin{align*}
&\mathbb{E}_{\cal T^{\rm new} \sim \P{(\cal T)}}\mathbb{E}_{\cal D \sim \cal T^{\rm new}}\mathbb{E}_{\{\x,y\}\sim \cal T^{\rm new}}\squarebrack{\round{\x^\top \hat{\beta}_{\rm Bayes}({\cal D})-y}^2}\\
\le& 12\delta+\sum_{i=1}^kp_i\esqnorm{\w_i-\hat{\w}_i}+\sum_{i=1}^kp_is_i^2
\end{align*}
Replacing $12\delta$ by $\delta$ concludes the proof.
\end{proof}

\section{Proof of Remark~\ref{thm:lb-predict}}
\label{sec:lb-predict-proof} 

We construct a worst case example and analyze the expected error of the Bayes optimal predictor. 
We choose $s_i = \sigma$, $p_i = 1/k$, and $\w_i = \round{\Delta/\sqrt{2}} \e_i $  for all $i\in [k]$. 
Given a new task with $\tau$  training examples, 
we assume Gaussian input $\x_j \sim {\cal N}(\zero,\I_d)\in \R^d$, 
and Gaussian noise $y_j = \beta^\top \x_j + \eps_j \in \R$ with $\epsilon_j\sim \N(0,\sigma^2)$ i.i.d. for all $j\in [ \tau ]$. 
Denote the true model parameter by $\beta = \w_{h}$ for some $h\in[k]$, and the Bayes optimal estimator is 
$$
\hat{\beta} = \squarebrack{\sum_{i=1}^k L_i}^{-1}\sum_{i=1}^k L_i\w_i,
$$
where $L_i \coloneqq \exp\round{-\frac{1}{2\sigma^2}\sum_{j=1}^\tau (y_j-\w_i^\top \x_j)^2}$. The squared $\ell_2$ error is lower bounded by 
\begin{align}
\esqnorm{\hat{\beta}-\w_h}
&\;\ge\; \esqnorm{\squarebrack{\sum_{i=1}^k L_i}^{-1}\sum_{i\in[k]\setminus\{h\}} L_i \w_h}\nonumber\\
&\;=\;  \frac{\Delta^2\round{\sum_{i\in[k]\setminus\{h\}}L_i/L_h}^2}{2\round{1+\sum_{i\in[k]\setminus\{h\}}L_i/L_h}^2}\label{eqn:lb:sqn-err}
\end{align}

Let us define $l_i = \log L_i$, which is 
\begin{align*}
l_i = &-\frac{1}{2\sigma^2}\sum_{j=1}^{ \tau }   \round{y_j-\x_j^\top{\w_i}}^2 \\
=& - \frac{1}{2\sigma^2} \sum_{j=1}^{ \tau }  \round{\eps_j+\x_j^\top(\w_h-{\w_i})}^2
\end{align*}

Notice that for all $i \in [k] \setminus \{h\}$, $\E{l_i} = -\frac{\tau}{2}\round{1+\nicefrac{\Delta^2}{\sigma^2}}$. Using Markov's inequality and the fact that $l_i\le 0$, we have that for each fixed $i \in [k] \setminus \{h\}$,
\begin{align*}
    \Prob{\; l_i \ge 3 \E{l_i}\; }\;\geq\; 2/3\;.
\end{align*} 
For each $i \in [k] \setminus \{h\}$, define an indicator random variable $I_i = \indicator{l_i \ge  3 \E{l_i} }$. The expectation is lower bounded by
\begin{align*}
\E{ \sum_{ i \in [k] \setminus \{h\} } I_i} \ge \frac{2}{3}(k-1)\; .
\end{align*}
The expectation is  upper bounded by
\begin{align*}
\E{ \sum_{ i \in [k] \setminus \{h\} } I_i } \le \Prob{ \sum_{ i \in [k] \setminus \{h\} } I_i \ge \frac{k-1}{3} } \cdot (k-1)\\
+\round{1-\Prob{ \sum_{ i \in [k] \setminus \{h\} } I_i \ge \frac{k-1}{3}}}\cdot \frac{k-1}{3}.
\end{align*}
Combining the above two bounds together, we have
\begin{align*}
\Prob{ \sum_{ i \in [k] \setminus \{h\} } I_i \ge \frac{k-1}{3} } \ge 1/2.
\end{align*}
Hence with probability at least $1/2$,
\begin{align*}
\sum_{i\in [k] \setminus \{h\} }e^{l_i-l_h}& \ge
\sum_{i\in [k] \setminus \{h\} }e^{l_i}
     \;\ge \sum_{i\in [k] \setminus \{h\} } I_i e^{3\E{l_i}}\\
    & \ge \frac{k-1}{3} e^{-\frac{3\tau}{2}\round{1+\nicefrac{\Delta^2}{\sigma^2}}}
    \;,
\end{align*}
which implies that Eq.~\eqref{eqn:lb:sqn-err} is greater than $\Delta^2/8$. Hence the expected $\ell_2$ error of the Bayes optimal estimator is ${\mathbb E}_{x,\eps}{\squarebrack{\round{\hat{y}-y}^2}}= {\mathbb E}\squarebrack{\Big(\round{\beta-\hat{\beta}}^\top \x + \eps\Big)^2}=\esqnorm{\beta - \hat\beta} + \sigma^2 = \Delta^2/8 + \sigma^2$.

\section{Technical definitions and facts}

\begin{definition}[Sub-Gaussian random variable]
    A random variable $X$ is said to follow a sub-Gaussian distribution if there exists a constant $K>0$ such that
    \[
    \Prob{\abs{X}>t} \leq 2\exp\round{-t^2/K^2}\qquad\forall\ t\geq 0.
    \]
\end{definition}
\begin{definition}[Sub-exponential random variable]
    A random variable $X$ is said to follow a sub-exponential distribution if there exists a constant $K>0$ such that
    \[
    \Prob{\abs{X}>t} \leq 2\exp\round{-t/K}\qquad\forall\ t\geq 0.
    \]
\end{definition}
\begin{definition}[Sub-exponential norm]
The sub-exponential norm of a random variable $X$ is defined as
\[
\norm{X}_{\psi_1} \coloneqq \sup\limits_{p\in\Natural} p^{-1}\round{\E{\abs{X}^p}}^{1/p}.
\]
A random variable is sub-exponential if its sub-exponential norm is finite.
\end{definition}

\begin{fact}[Gaussian and sub-Gaussian 4-th moment condition]\label{fact:gauss-4}
Let $\vvec$ and $\uvec$ denote two fixed vectors, we have
\begin{align*}
\Exp{ \x \sim {\cal N}(\zero,\I) }{ \round{ \vvec^\top \x }^2\round{ \uvec^\top \x }^2 } = \| \uvec \|_2^2 \cdot \| \vvec \|_2^2 + 2 \langle \uvec , \vvec \rangle^2.
\end{align*}
If $\x$ is a centered sub-Gaussian random variable with identity second moment, then 
\begin{align*}
    \Exp{}{ \round{ \vvec^\top \x }^2\round{ \uvec^\top \x }^2 } = \BigO{ \| \uvec \|_2^2 \cdot \| \vvec \|_2^2}.
\end{align*}
\end{fact}
\begin{proposition}[Matrix Bernstein inequality, Theorem 1.6.2 in \cite{tropp2015introduction}]\label{prop:matrix-bernstein}
Let $\mathbf{S}_1,\ldots,\mathbf{S}_n$ be independent, centered random matrices with common dimension $d_1\times d_2$, and assume that each one is uniformly bounded $\E{\mathbf{S}_k} = 0$ and $\enorm{\mathbf{S}_k} \le L\ \forall\ k = 1, \ldots, n$.

Introduce the sum
\[
\Z \coloneqq \sum_{k=1}^n \mathbf{S}_k
\]
and let $v(\Z)$ denote the matrix variance statistic of the sum:
\begin{align*}
v(\Z) \coloneqq \max\curly{\enorm{\E{\Z\Z^\top}}, \enorm{\E{\Z^\top\Z}}}
\end{align*}
Then
\begin{align*}
\Prob{\enorm{\Z} \ge t} \le (d_1 +d_2)\exp\curly{\frac{-t^2/2}{v(\Z)+Lt/3}}
\end{align*}
for all $t \ge 0$.
\end{proposition}
\begin{fact}[Hoeffding's inequality~\cite{hoeffding1963}]\label{fact:hoeffding}
Let $X_1,\ldots, X_n$ be independent random variables with bounded interval $0 \le X_i \le 1$. Let $\overline{X} = \frac{1}{n} \sum_{i=1}^n X_i$. Then
\begin{align*}
\Prob{\abs{ \overline{X}  - \E{ \overline{X} }}\ge z}\le 2\exp\curly{-2nz^2}.
\end{align*}
\end{fact}
\begin{proposition}[$\ell_\infty$ deviation bound of multinomial distributions]\label{prop:mult-con}
Let $\p = \{p_1,\ldots,p_k\}$ be a vector of probabilities (i.e. $p_i\ge 0$ for all $i\in [k]$ and $\sum_{i=1}^kp_i = 1$). Let $\x \sim {\rm multinomial} (n, \p)$ follow a multinomial distribution with $n$ trials and probability $\p$. Then with probability $1-\delta$, for all $i\in [k]$,
$$
\abs{\frac{1}{n}x_i-p_i} \le \sqrt{\frac{3\log(k/\delta)}{n}p_i},
$$
which implies
$$
\norm{\frac{1}{n}\x-\p}_{\infty} \le \sqrt{\frac{3\log(k/\delta)}{n}}.
$$
for all $i\in\squarebrack{k}$.
\end{proposition}
\begin{proof}
For each element $x_i$, applying Chernoff Bound~\ref{fact:chernoff} with $z = \sqrt{\frac{3\log (k/\delta)}{n\E{\overline{X}}}}$ and taking a union bound over all $i$, we get 
$$
\abs{\frac{1}{n}x_i-p_i}\le \sqrt{\frac{3\log(k/\delta)p_i}{n}}.
$$
for all $i\in\squarebrack{k}$.\qedhere
%Applying Hoeffding's inequality (Fact~\ref{fact:hoeffding}) for each element $x_i$ of $\x$ and taking a union bound over all $k$ coordinates yields the proposition.

\end{proof}
\begin{fact}[Chernoff Bound]\label{fact:chernoff}
Let  $X_1, \ldots, X_n$ be independent Bernoulli random variables.  Let $\overline{X} = \frac{1}{n} \sum_{i=1}^n X_i$. Then for all $0 < \delta \le 1$
\begin{align*}
\Prob{\abs{ \overline{X}  - \E{ \overline{X} }}\ge z \E{ \overline{X}}}\le \exp\curly{-z^2n\E{ \overline{X}}/3}.
\end{align*}
\end{fact}
\begin{proposition}[High probability bound on the error of random design linear regression]\label{prop:ols-error}
Consider the following linear regression problem where we are given $n$ i.i.d. samples
\[
\x_i \sim D~,~y_i = \beta^\top \x_i + \eps_i~,~i\in [n] 
\]
where $D$ is a $d$-dimensional ($d<n$) sub-Gaussian distribution with constant sub-gaussian norm, $\E{\x_i}=0$, $\E{\x_i\x_i^\top}=\I_d$, and $\eps_i$ is a sub-gaussian random variable and satisfies $\E{\eps_i} = 0$, $\E{\eps_i^2} = \sigma^2$.
\begin{enumerate}
\item Then, with probability $1-\delta$, the ordinary least square estimator $\hat{\beta} \coloneqq \argmin_{\w} \sum_{i=1}^n\round{y_i-\w^\top\x_i}^2$ satisfies
\[
\esqnorm{\hat{\beta}-\beta} \le \BigO{\frac{\sigma^2(d+\log(1/\delta))}{n}}.
\]
\item Define the estimator of the noise $\hat{\sigma}^2$ as
\[
\hat{\sigma}^2 \coloneqq \frac{1}{n-d}\sum_{i=1}^n\round{y_i-\hat{\beta}^\top \x_i}^2.
\]
Then with probability $1-\delta$, it holds that
\[
|\hat{\sigma}^2 - \sigma^2|\le \frac{\log(1/\delta)}{\sqrt{n-d}}\sigma^2.
\]
\end{enumerate}
\end{proposition}
\begin{proof}
\citep[Remark~12]{hsu2012random} shows that in the setting stated in the proposition, with probability $1-\exp(-t)$, it holds that the least square estimator 
\[
\esqnorm{\hat{\beta}-\beta} \le  \BigO{\frac{\sigma^2\round{d+2\sqrt{dt}+2t}}{n}}+o\round{\frac{1}{n}}.
\]
This implies that with probability $1-\delta$, it holds that 
\[
\esqnorm{\hat{\beta}-\beta} =   \BigO{\frac{\sigma^2(d+\log(1/\delta))}{n}}.
\]
To prove the second part of the proposition, we first show that $\hat{\sigma}^2$ is an unbiased estimator for $\sigma^2$ and then apply Hanson-Wright inequality to show the concentration. Define vector $\y \coloneqq (y_1,\ldots, y_n)$, $\bm{\eps} \coloneqq (\eps_1, \ldots, \eps_n)$ and matrix $\X \coloneqq \begin{bmatrix}
\x_1,
\ldots,
\x_n
\end{bmatrix}^\top$. Notice that 
\begin{align*}
\E{\hat{\sigma}^2} &= \frac{1}{n-d}\E{\sum_{i=1}^n\round{y_i-\hat{\beta}^\top \x_i}^2}\\
&= \frac{1}{n-d}\E{\bm{\eps}^\top\round{\I_n - \X\round{\X^\top\X}^{-1}\X^\top}\bm{\eps}}\\
& = \inv{n-d}\E{\tr\squarebrack{\I_n - \X\round{\X^\top\X}^{-1}\X^\top}} = \sigma^2,
\end{align*}
where the last equality holds since $\X\round{\X^\top\X}^{-1}\X^\top$ has exactly $d$ eigenvalues equal to $1$ almost surely. For a fixed $\X$ with rank $d$, by Hanson-Wright inequality~\citep[Theorem~6.2.1]{vershynin2018high}, it holds that 
\[
\Prob{\abs{\hat{\sigma}^2-\sigma^2}\ge z} \le 2\exp\curly{-c\min\curly{(n-d)z^2/\sigma^4,(n-d)z/\sigma^2}},
\]
which implies that with probability $1-\delta$
\[
\abs{\hat{\sigma}^2-\sigma^2}=\BigO{ \frac{\log(1/\delta)}{\sqrt{n-d}}\sigma^2}.\qedhere
\]
\end{proof}

\section{Simulations}\label{sec:simulations}
We set $d=8k$, $\p=\one_k/k$, $\s=\one_k$, and $\Pcal_\x$ and $\Pcal_\epsilon$ are standard Gaussian distributions.

\subsection{Subspace estimation}\label{sec:sim_subspace}

We compute the subspace estimation error $\rho^{-1}\max_{i\in\squarebrack{k}}\enorm{\round{\U\U^\top-\I}\w_i}$ for various $(t_{L1},n_{L1})$ pairs for $k=16$ and present them in Table~\ref{tab:subspace_estimation}.

\begin{table}[ht]
    \caption{Error in subspace estimation for $k=16$, varying $n_{L1}$ \& $t_{L1}$.}
    \label{tab:subspace_estimation}
    \vskip 0.15in
    \begin{center}
    \begin{small}
    \begin{sc}
    \resizebox{0.48\textwidth}{!}{
    \def\arraystretch{1.1}
    \begin{tabular}{c||c c c c c c c }
    \toprule
        $(t_{L1},n_{L1})$ & $2^{14}$ & $2^{15}$ & $2^{16}$ & $2^{17}$ & $2^{18}$ & $2^{19}$ & $2^{20}$ \\
    \midrule
        $2^1$ & $0.652$ & $0.593$ & $0.403$ & $0.289$ & $0.195$ & $0.132$ & $0.101$ \\
        $2^2$ & $0.383$ & $0.308$ & $0.194$ & $0.129$ & $0.101$ & $0.069$ & $0.05$ \\
        $2^3$ & $0.203$ & $0.153$ & $0.099$ & $0.072$ & $0.052$ & $0.034$ & $0.03$ \\
    \bottomrule
    \end{tabular}}
    \end{sc}
    \end{small}
    \end{center}
    \vskip -0.1in
\end{table}

\subsection{Clustering}\label{sec:sim_clustering}
Given a subspace estimation error is $\sim0.1$, the clustering step is performed with $n_H=\max\curly{k^{3/2},256}$ tasks for various $t_H$. The minimum $t_H$ such that the clustering accuracy is above $99\%$ for at-least $1-\delta$ fraction of $10$ random trials is denoted by $t_{\rm min}(1-\delta)$. Figure~\ref{fig:clustering_tmin}, and Table~\ref{tab:clustering} illustrate the dependence of $k$ on $t_{\rm min}(0.5)$, and $t_{\rm min}(0.9)$.

\begin{figure}[ht]
    \vskip 0.2in
    \begin{center}
    \centerline{\includegraphics[width=0.5\linewidth]{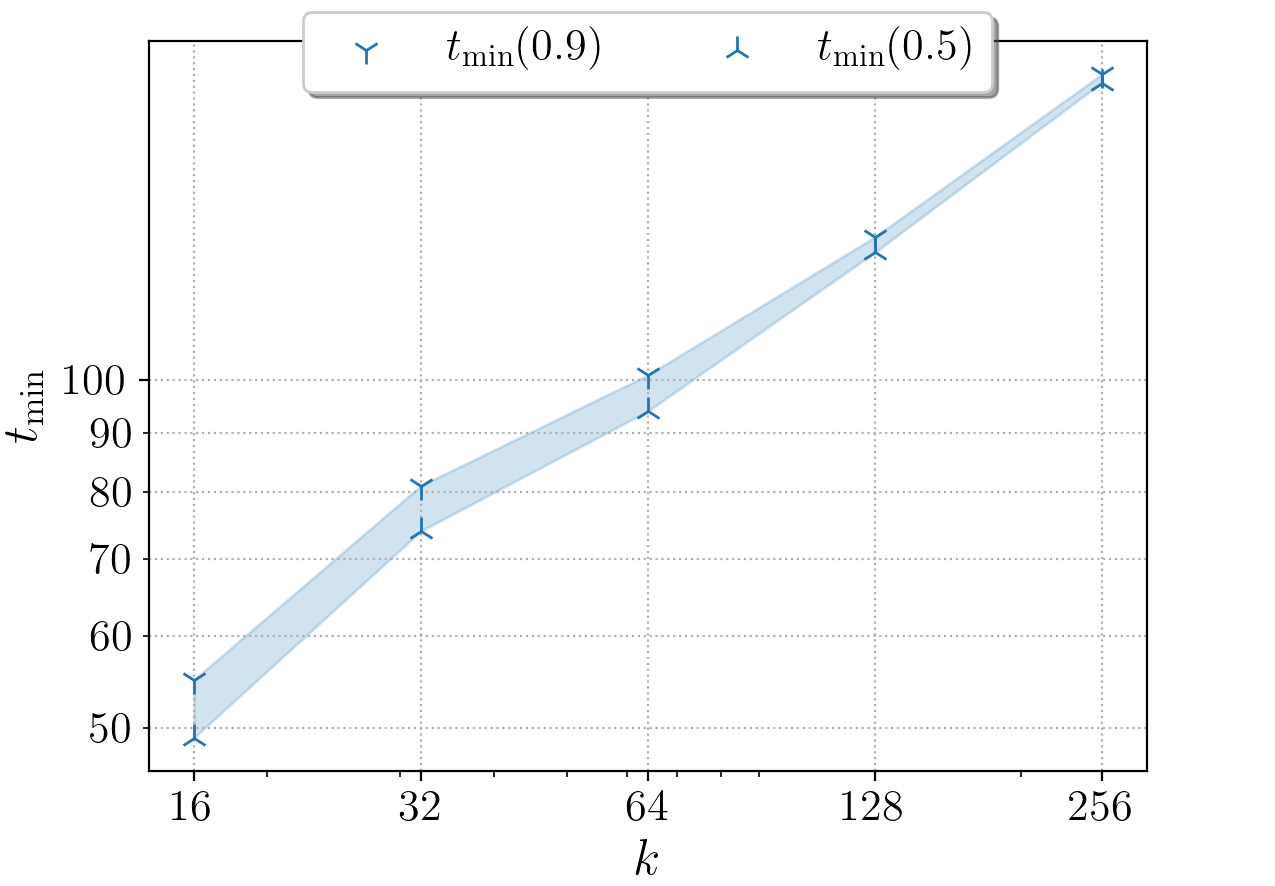}}
    \caption{$t_{\rm min}(0.9)$ and $t_{\rm min}(0.5)$ for various $k$}
    \label{fig:clustering_tmin}
    \end{center}
    \vskip -0.2in
\end{figure}

\begin{table}[ht]
    \caption{$t_{\rm min}$ for various $k$, for $99\%$ clustering w.h.p.}
    \label{tab:clustering}
    \vskip 0.15in
    \begin{center}
    \begin{small}
    \begin{sc}
    \def\arraystretch{1.1}
    \begin{tabular}{ c||c c c c c }
    \toprule
        $k$ & $16$ & $32$ & $64$ & $128$ & $256$ \\ \hline
        $t_{\rm min}(0.9)$ & $55$ & $81$ & $101$ & $133$ & $184$ \\
        $t_{\rm min}(0.5)$ & $49$ & $74$ & $94$ & $129$ & $181$ \\
    \bottomrule
    \end{tabular}
    \end{sc}
    \end{small}
    \end{center}
    \vskip -0.1in
\end{table}

\subsection{Classification and parameter estimation}\label{sec:sim_classification}
Given a subspace estimation error is $\sim 0.1$, and a clustering accuracy is $>99\%$, the classification step is performed on $n_{L2} = \max\curly{512,k^{3/2}}$ tasks for variour $t_{L2}\in\Natural$. The empirical mean of the classification accuracy is computed for every $t_{L2}$, and illustrated in Figure~\ref{fig:clustering_acc}. Similar to the simulations in the clustering step, $t_{\rm min}(1-\delta)$ is estimated such that the classification accuracy is above $99\%$ for at-least $1-\delta$ fraction times of $10$ random trials, and is illustrated in Table~\ref{tab:classification}. With $t_{L2} = t_{\rm min}(0.9)$, and various $n_{L2}\in\Natural$, the estimation errors of $\hat{\W}$, $\hat{\s}$, and $\hat{\p}$ are computed as the infimum of $\eps$ satisfying~\eqref{eq:estimation_err}, and is illustrated in Figure~\ref{fig:est_err}.

\begin{table}[ht]
    \caption{$t_{\rm min}$ for various $k$, for $99\%$ classification w.h.p.}
    \label{tab:classification}
    \vskip 0.15in
    \begin{center}
    \begin{small}
    \begin{sc}
    \def\arraystretch{1.1}
    \begin{tabular}{ c||c c c c c }
    \toprule
        $k$ & $16$ & $32$ & $64$ & $128$ \\ \hline
        $t_{\rm min}(0.9)$ & $31$ & $34$ & $36$ & $38$ \\
        $t_{\rm min}(0.5)$ & $28$ & $28$ & $34$ & $36$ \\
    \bottomrule
    \end{tabular}
    \end{sc}
    \end{small}
    \end{center}
    \vskip -0.1in
\end{table}

\begin{figure}[ht]
    \vskip 0.2in
    \begin{center}
    \centerline{\includegraphics[width=0.5\linewidth]{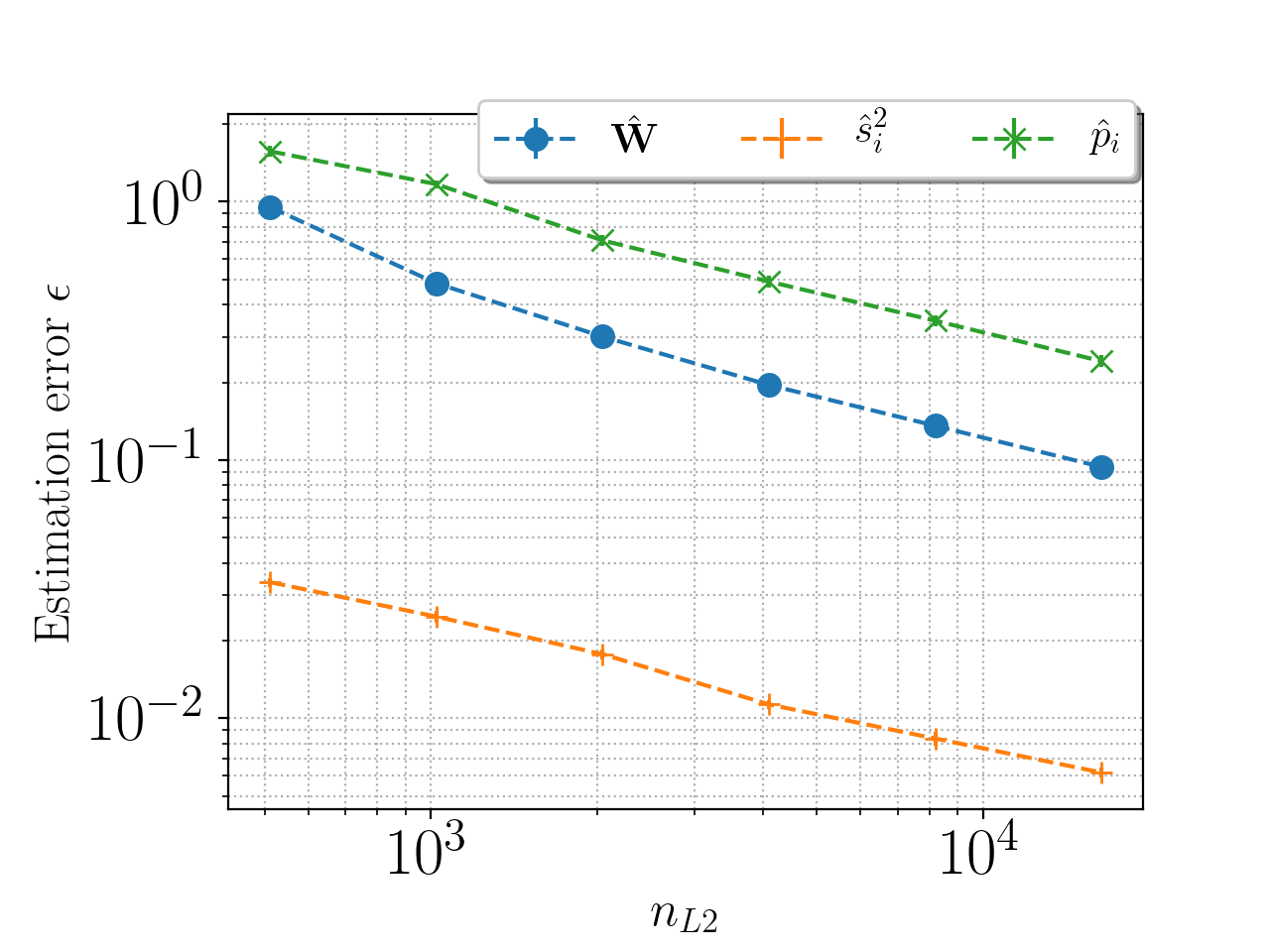}}
    \caption{Estimation errors for $k=32$.}
    \label{fig:est_err}
    \end{center}
    \vskip -0.2in
\end{figure}

\begin{figure}[ht]
    \vskip 0.2in
    \begin{center}
    \subfigure[$k=32$]{\label{fig:clustering_acc_32}\includegraphics[scale=0.32]{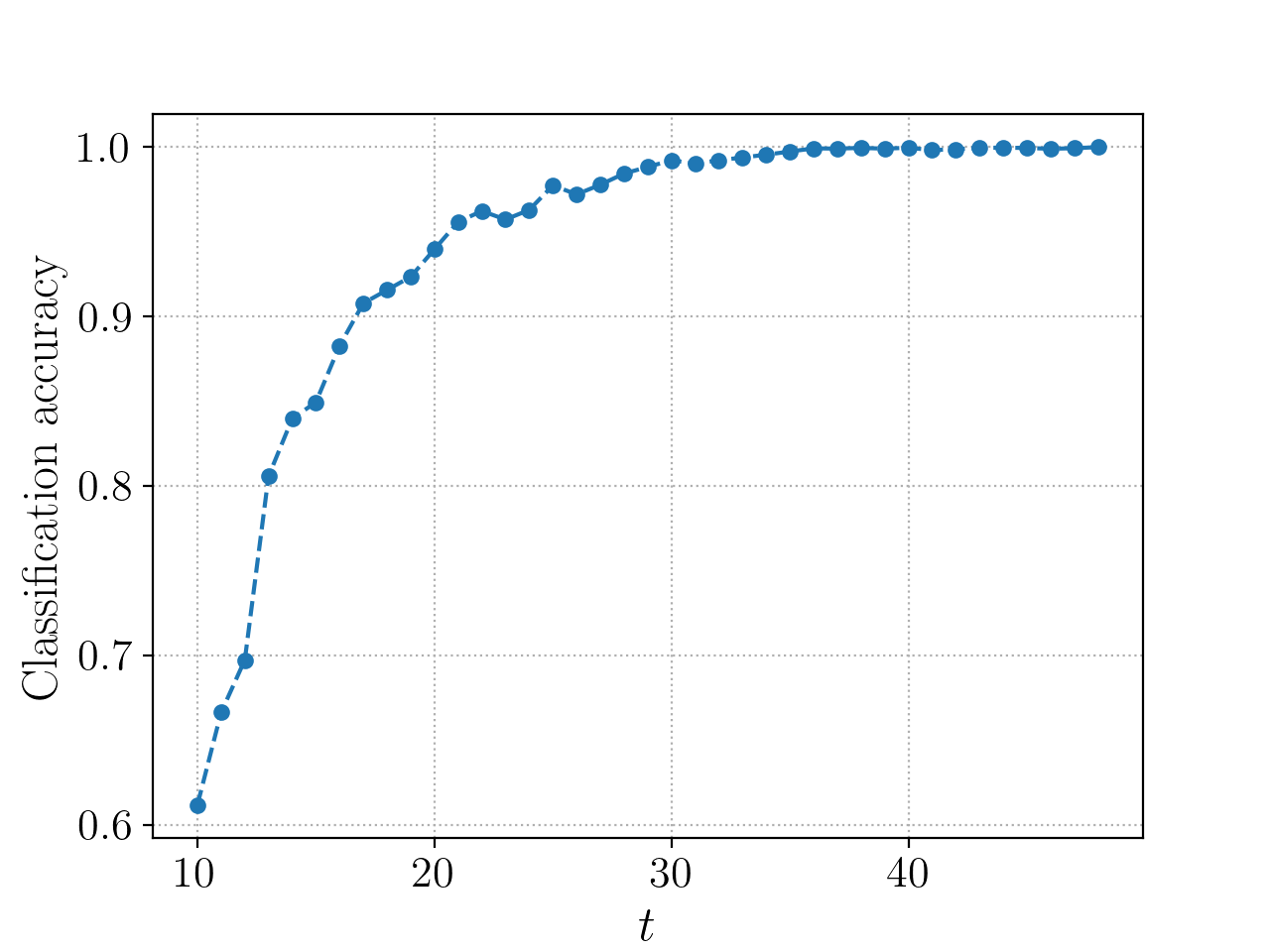}}
    \subfigure[$k=64$]{\label{fig:clustering_acc_64}\includegraphics[scale=0.32]{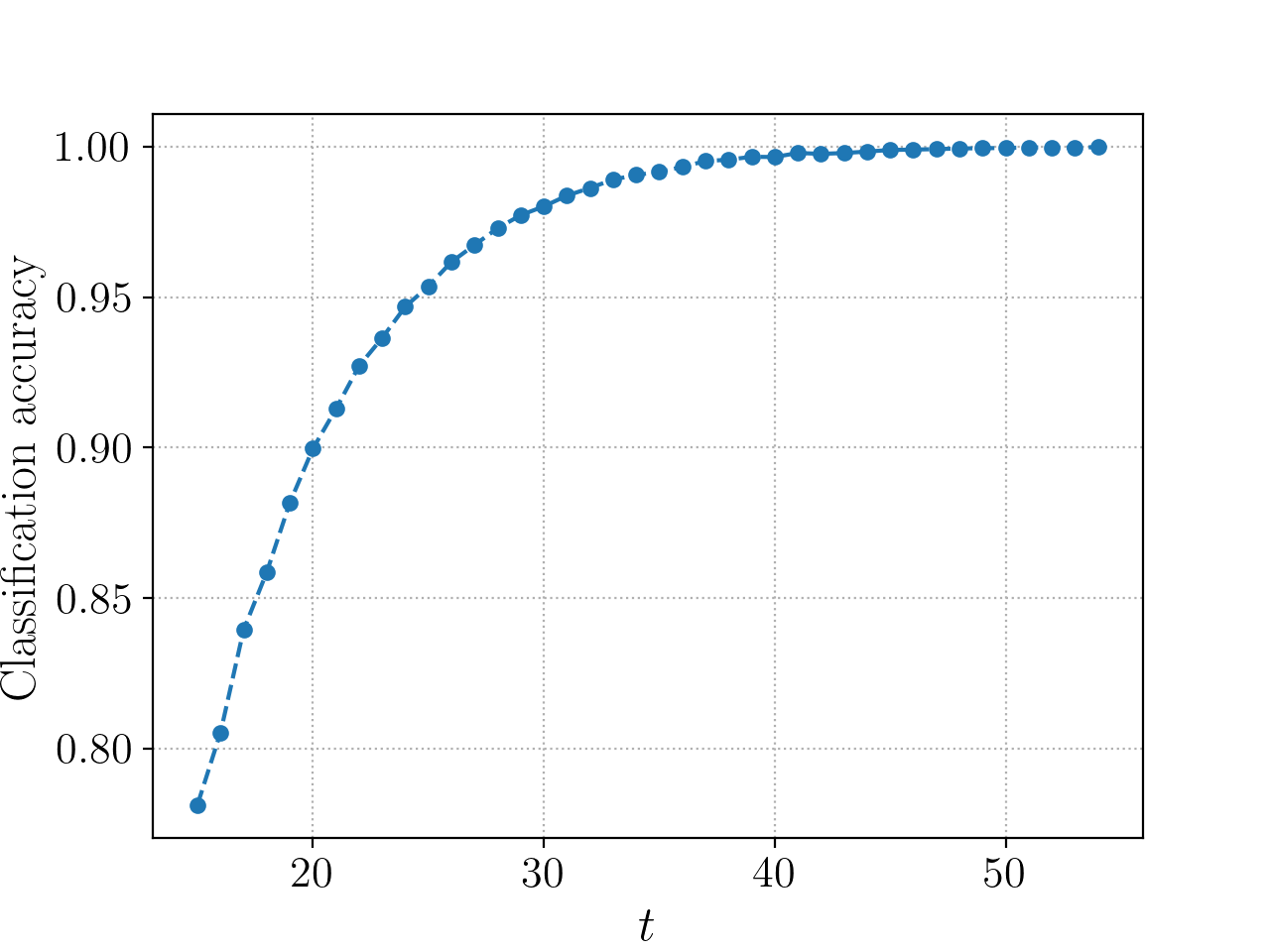}}
    \subfigure[$k=128$]{\label{fig:clustering_acc_128}\includegraphics[scale=0.32]{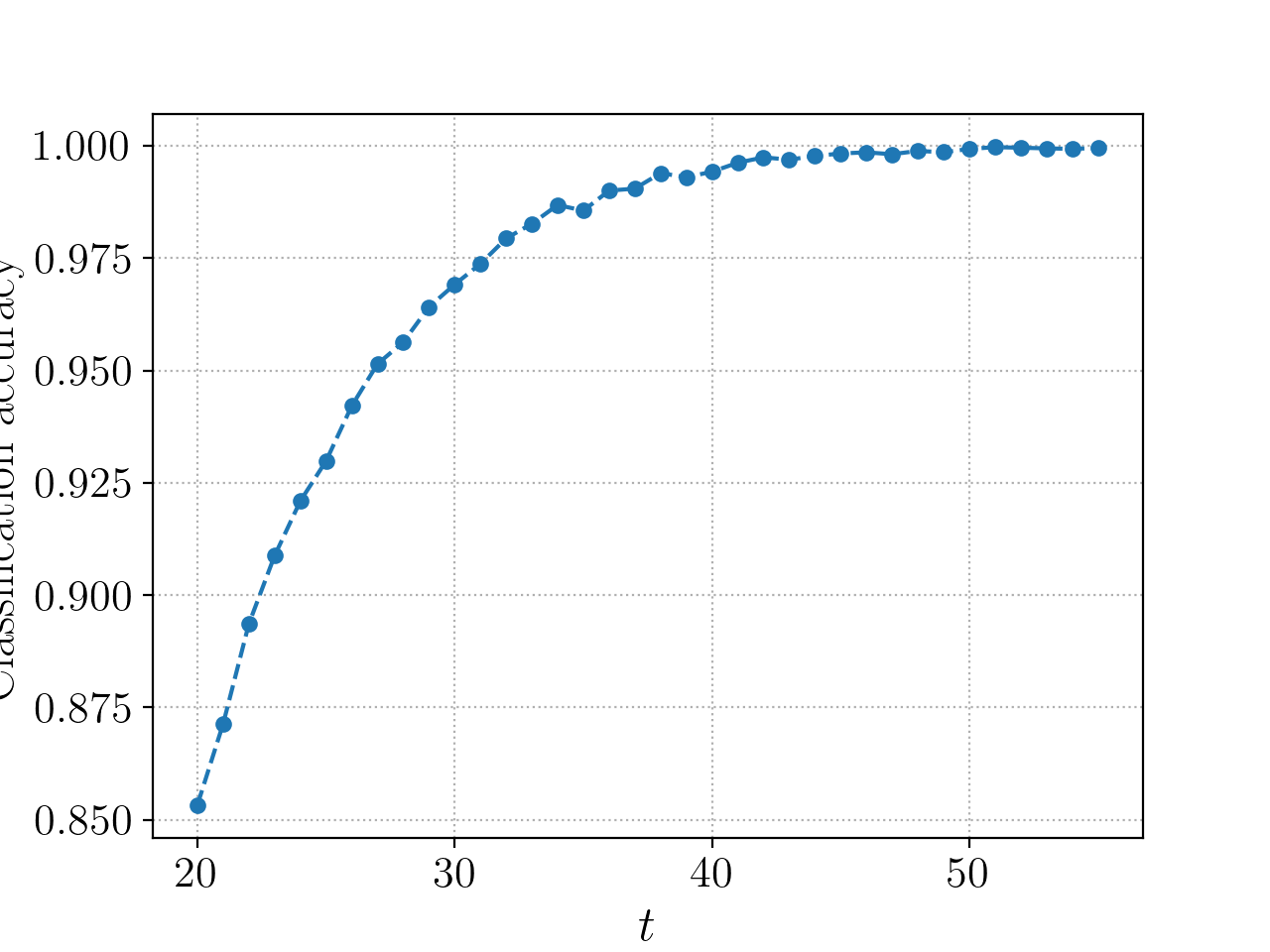}}
    \caption{Classification accuracies for various $k$}
    \label{fig:clustering_acc}
    \end{center}
    \vskip -0.2in
\end{figure}

\subsection{Prediction}
As a continuation of the simulations in this section, we proceed to the prediction step for $k=32$ and $d=256$. We use both the estimators: Bayes estimator, and the MAP estimator and illustrate the training and prediction errors in Figure~\ref{fig:pred}. We also compare the prediction error with the vanilla least squares estimator if each task were learnt separately to contrast the gain in meta-learning.

\subsection{Comparison for parameter estimation against Expectation Maximization (EM) algorithm}
For fair comparisons, we consider our meta dataset for $k=32$, and $d=256$ to jointly have $n_{L1}$ tasks with $t_{L1}$ examples, $n_{H}$ tasks with $t_{H}$ examples, and $n_{L2}$ tasks with $t_{L2}$ examples as were used in Section~\ref{sec:sim_classification}. We observe that the convergence of EM algorithm is very sensitive to the initialization, thus we investigate the sensitivity with the following experiment. We initialize $\W^{(0)} = \Pcal_{B_{2,d}(\zero,1)}\round{\W+\Z}$, where $Z_{i,j}\sim\Ncal(0,\gamma^2)\ \forall\ i\in\squarebrack{d}, j\in\squarebrack{k}$, $\s = \abs{\q}$, where $\q \sim \Ncal\round{\s,0.1\I_k}$, and $\p^{(0)} = \abs{\z}/\norm{\z}_1$ where $\z\sim\Ncal\round{\p,\I_k/k}$. $\Pcal_{\Xcal}(\cdot)$ denotes the projection operator that projects each column of its argument on set $\Xcal$. We observe that EM algorithm fails to converge for $\gamma^2\geq 0.5$ for this setup unlike our algorithm.

\end{document}